\documentclass[10pt]{article}

\usepackage[utf8]{inputenc}
\usepackage[T1]{fontenc}

\usepackage{epsf}
\usepackage{amsmath}

\allowdisplaybreaks

\usepackage[showframe=false]{geometry}
\usepackage{changepage}

\usepackage{epsfig}
\usepackage{amssymb}

\usepackage{amsthm}
\usepackage{setspace}
\usepackage{cite}
\usepackage{mcite}

\usepackage{algorithmic}  
\usepackage{algorithm}

\usepackage{shadow}
\usepackage{fancybox}
\usepackage{fancyhdr}

\usepackage{color}
\usepackage[usenames,dvipsnames,svgnames,table]{xcolor}
\newcommand{\bl}[1]{\textcolor{blue}{#1}}

\definecolor{mypurple}{rgb}{.4,.0,.5}
\newcommand{\prp}[1]{\textcolor{mypurple}{#1}}

\usepackage[hyphens]{url}

\usepackage[colorlinks=true,
            linkcolor=black,
            urlcolor=blue,
            citecolor=purple]{hyperref}

\usepackage{breakurl}

\def\y{{\bf y}}

\def\x{{\bf x}}

\def\x{{\mathbf x}}

\def\x{{\bf x}}
\def\y{{\bf y}}
\def\z{{\bf z}}

\def\h{{\bf h}}

\def\cL{{\mathcal L}}

\def\cA{{\mathcal A}}

\def\cG{{\mathcal G}}

\def\be{\begin{equation}}
\def\ee{\end{equation}}
\def\ba{\left[\begin{array}}
\def\ea{\end{array}\right]}

\def\t{{\bf t}}

\def\x{{\bf x}}
\def\y{{\bf y}}
\def\z{{\bf z}}

\def\1{{\bf 1}}

\def\g{{\bf g}}
\def\0{{\bf 0}}

\def\erf{\mbox{erf}}
\def\erfc{\mbox{erfc}}

\def\mR{{\mathbb R}}

\def\mE{{\mathbb E}}

\def\mP{{\mathbb P}}

\def\lp{\left (}
\def\rp{\right )}

\def\y{{\bf y}}

\def\x{{\bf x}}

\def\x{{\mathbf x}}

\def\x{{\bf x}}
\def\y{{\bf y}}
\def\z{{\bf z}}

\def\h{{\bf h}}

\def\be{\begin{equation}}
\def\ee{\end{equation}}
\def\ba{\left[\begin{array}}
\def\ea{\end{array}\right]}

\def\t{{\bf t}}

\def\x{{\bf x}}
\def\y{{\bf y}}
\def\z{{\bf z}}

\def\({\left (}
\def\){\right )}

\def\1{{\bf 1}}

\def\g{{\bf g}}
\def\0{{\bf 0}}

\def\cY{{\mathcal Y}}

\usepackage{xcolor}
\usepackage{color}

\definecolor{darkgreen}{rgb}{0, 0.4,0}

\definecolor{purplebrown}{rgb}{0.5,0.1,0.6}

\definecolor{ultclupcol}{rgb}{0.1,0.5,0.5}

\definecolor{mytrycolor}{rgb}{0.5,0.7,0.2}

\definecolor{ultclupcola}{rgb}{.5,0,.5}

\definecolor{shadebrown}{rgb}{0.1,0.1,0.9}
\definecolor{lightblue}{rgb}{0.2,0,1}

\usepackage{fancybox}
\usepackage{graphicx}
\usepackage{epstopdf}
\usepackage{epsfig}
\usepackage{wrapfig}
\usepackage{subfigure}

\usepackage{xcolor}
\usepackage{tcolorbox}
\tcbuselibrary{skins}

\newtcbox{\xmybox}{on line,
arc=7pt,
before upper={\rule[-3pt]{0pt}{10pt}},boxrule=0pt,
boxsep=0pt,left=6pt,right=6pt,top=0pt,bottom=0pt,enhanced, coltext=blue, colback=white!10!yellow}

\newtcbox{\xmyboxa}{on line,
arc=7pt,
before upper={\rule[-3pt]{0pt}{10pt}},boxrule=0pt,
boxsep=0pt,left=6pt,right=6pt,top=0pt,bottom=0pt,enhanced, colback=white!10!yellow}

\newtcbox{\xmyboxb}{on line,
arc=7pt,
before upper={\rule[-3pt]{0pt}{10pt}},boxrule=1pt,colframe=darkgreen!100!blue,
boxsep=0pt,left=6pt,right=6pt,top=0pt,bottom=0pt,enhanced, colback=white!10!yellow}

\newtcbox{\xmyboxc}{on line,
arc=7pt,
before upper={\rule[-3pt]{0pt}{10pt}},boxrule=.7pt,colframe=blue!100!blue,
boxsep=0pt,left=6pt,right=6pt,top=0pt,bottom=0pt,enhanced, coltext=blue, colback=white!10!yellow}

\newtcbox{\xmytboxa}{on line,
arc=7pt,
before upper={\rule[-3pt]{0pt}{10pt}},boxrule=.0pt,colframe=pink!50!yellow,
boxsep=0pt,left=6pt,right=6pt,top=0pt,bottom=0pt,enhanced, coltext=white, colback=blue!40!red}

\newtcbox{\xmytboxb}{on line,
arc=7pt,
before upper={\rule[-3pt]{0pt}{10pt}},boxrule=.0pt,colframe=pink!50!yellow,
boxsep=0pt,left=6pt,right=6pt,top=0pt,bottom=0pt,enhanced, coltext=white, colback=white!40!green}

\makeatletter
\newcommand\subsubsubsection{\@startsection{paragraph}{4}{\z@}{-2.5ex\@plus -1ex \@minus -.25ex}{1.25ex \@plus .25ex}{\normalfont\normalsize\bfseries}}
\newcommand\subsubsubsubsection{\@startsection{subparagraph}{5}{\z@}{-2.5ex\@plus -1ex \@minus -.25ex}{1.25ex \@plus .25ex}{\normalfont\normalsize\bfseries}}
\makeatother

\newtheorem{theorem}{Theorem}

\newtheorem{lemma}{Lemma}

\begin{document}

\begin{singlespace}

\title{Deep ReLU networks -- injectivity capacity upper bounds  
}
\author{
\textsc{Mihailo Stojnic
\footnote{e-mail: {\tt flatoyer@gmail.com}} }}
\date{}
\maketitle

\centerline{{\bf Abstract}} \vspace*{0.1in}

We study deep ReLU feed forward neural networks (NN) and their injectivity abilities.  The main focus is on \emph{precisely} determining the so-called injectivity capacity. For any given hidden layers architecture, it is defined as the minimal ratio between number of network's outputs and inputs which ensures unique recoverability of the input from a realizable output. A strong recent progress in precisely studying single ReLU layer injectivity properties is here moved to a deep network level. In particular, we develop a program that connects deep $l$-layer  net injectivity to an $l$-extension of the  $\ell_0$ spherical perceptrons, thereby massively generalizing an isomorphism between studying single layer injectivity and the capacity of the so-called (1-extension) $\ell_0$ spherical perceptrons  discussed in \cite{Stojnicinjrelu24}. \emph{Random duality theory} (RDT) based machinery is then created and utilized to statistically handle properties of the extended $\ell_0$ spherical perceptrons and implicitly of the deep ReLU NNs.  A sizeable set of numerical evaluations is conducted as well to put the entire RDT machinery in  practical use. From these we observe a rapidly decreasing tendency in needed layers' expansions, i.e., we observe a rapid \emph{expansion saturation effect}. Only  $4$ layers of depth are sufficient to closely approach level of no needed expansion -- a result that fairly closely resembles observations made in practical experiments and that has so far remained completely untouchable by any of the existing mathematical methodologies.

\vspace*{0.25in} \noindent {\bf Index Terms: Injectivity; Deep ReLU networks; Random duality}.

\end{singlespace}

\section{Introduction}
\label{sec:back}

An avalanche of research in machine learning (ML) and neural networks (NN) over the last decade produced some of the very best scientific breakthroughs. These include developments of both excellent algorithmic methodologies  as well as their accompanying theoretical justifications. For almost all of them a superior level of understanding of underlying mathematical principles is needed. In this paper we study such a principle called injectivity and discuss how its presence/absence impacts/limits functioning of neural nets.

Just by its definition, the functional injectivity plays a critical role in studying inverse problems and is in a direct correspondence  with their well- or ill-posedness. It comes as a no surprise that recent theoretical and practical studying of (nonlinear) inverse problems via neural nets heavily relies on the associated injectivities (see, e.g., \cite{ArridgeMOS19,BoraJPD17,KothariKHD21,HandLV18,DharGE18,LeiJDD19,EstrachSL14,DaskalakisRZ20a,Stojnicinjrelu24}).  Consequently, many aspects od injectivity  gained strong interest in recent years including  studying Lipshitzian/stability  properties,  \cite{FazlyabRHMP19,GoukFPC21,JordanD20}, role of the injective ReLU nets in manifold densities and approximative maps \cite{PuthawalaKLDH22,PuthawalaLDH22,RossC21},  algorithmic approaches to  generative models \cite{FletcherRS18,LeiJDD19,RomanoEM17,HeWJ17,MardaniSDPMVP18,SchniterRF16,RanganSF17,ShahH18,HegdeWB07,HandLV18,WuRL19},  deep learning compressed sensing/phase retrieval \cite{BoraJPD17,KothariKHD21,LeiJDD19,DaskalakisRZ20a,DharGE18,HandLV18,MardaniSDPMVP18,HHHV18},  and random matrix - neural networks connections \cite{MBBDN23,PenningtonW17,LouartLC17}.

As discussed in many of the above works, characterizing analytically  injectivity of a whole network is not an easy mathematical problem. It typically relates to the so-called injectivity capacity defined as the minimal ratio of the network's number of outputs and inputs for which a \emph{unique} generative input produces a realizable output. Despite a host of technical difficulties, in addition to excellent practical implementations, strong accompanying analytical results are obtained in many of the above works as well. They usually relate to the so-called \emph{qualitative} performance characterizations which typically  give correct dimensional orders and, as such, provide an intuitive guidance for building networks' architectures. Since our interest is in more precise, i.e.,  \emph{quantitative} performance characterizations, results from \cite{Pal21,Clum22,MBBDN23,PuthawalaKLDH22,Stojnicinjrelu24} are more closely related to ours and we discuss them throughout the presentation after the introduction of necessary technical preliminaries.

\section{Mathematical preliminaries, related work, and contributions}
 \label{sec:mathsetup}

For any given positive integer $l$, consider  sequences of positive integers  $m_0,m_1,m_2,m_3,\dots,m_l$, matrices $A^{(i)}\in \mR^{m_i\times m_{i-1}}$, and real maps $f_{g_i}(\cdot): \mR^{m_i}\rightarrow\mR^{m_i}$. The following \emph{nonlinear} system of equations will be the main mathematical object of our study
\begin{eqnarray}
\bar{\y}^{(l)}=  f_{g_l} \(A^{(l)} \dots  \(A^{(4)} f_{g_3} \(A^{(3)} f_{g_2}\( A^{(2)} f_{g_1}(A^{(1)}\bar{\x}) \)\)\)\). \label{eq:ex1}
\end{eqnarray}
Adopting the convention $n\triangleq m_0$, one has that $A^{(i)}$'s  are linearly transformational system matrices, $f_{g_i}(\cdot)$ are (nonlinear) system functions, and $\bar{\x}\in\mR^n$ is a generative input vector to be recovered. A generic nature of our presentation will ensure that both developed methodologies and obtained results  can be utilized in conjunction with a host of different activation functions $f_{g_i}(\cdot)$. To ensure neatness of the exposition, we consider the so-called componentwise activation functions that act  in the same manner on each coordinate of their vector argument. Given their importance in studying and practical utilization of deep neural networks (DNN), we consider  ReLUs as concrete activation examples
\begin{eqnarray}
 f_{g_i}(\x)=\max(\x,0), \label{eq:ex1a0}
\end{eqnarray}
 with $\max$ being applied componentwise. After setting
\begin{eqnarray}
{\mathcal A}_{1:l} \triangleq [A^{(1)},A^{(2)},\dots,A^{(l)}], \label{eq:ex1a0a0}
\end{eqnarray}
it is then not that difficult to see that the nonlinear system from (\ref{eq:ex1}) becomes the so-called
\begin{equation}
\hspace{-0in}\bl{\textbf{\emph{Deep ReLU system ${\mathcal N}_{1:l}$:}}} \hspace{.2in} \bar{\y}^{(l)}=\max\(A^{(l)}\dots \(A^{(2)}\max\(A^{(1)}\bar{\x},0 \) ,0 \)  ,0 \) \triangleq \bar{f}_{nn}(\bar{\x};{\mathcal A}_{1:l}). \label{eq:ex1a2}
\end{equation}
We particularly focus on mathematically typically the most challenging  \emph{linear} (or, as often called, \emph{proportional}) high-dimensional regimes with $i$-th layer \emph{absolute} expansion coefficients
 \begin{eqnarray}
\alpha_i  \triangleq    \lim_{m_0\rightarrow \infty} \frac{m_i}{m_{0}}, i=1,2,\dots,l. \label{eq:ex15}
\end{eqnarray}
These are closely connected to  \emph{relative} expansion coefficients
 \begin{eqnarray}
\zeta_i  \triangleq    \lim_{m_0\rightarrow \infty} \frac{m_i}{m_{i-1}} = \frac{\alpha_i}{\alpha_{i-1}}, i=1,2,\dots,l. \label{eq:ex15a0}
\end{eqnarray}
Clearly, the absolute expansion coefficients relate to the expansion of the whole network, whereas the relative ones relate to the expansions within each of the layers. The system in (\ref{eq:ex1a2}) is a mathematical description of a feed forward ReLU DNN with input $\bar{\x}$, output $\bar{\y}$, and the weights of the gates in the $i$-th hidden layer being the rows of $A^{(i)}$.

Given the growing popularity of machine learning (ML) and neural networks (NN) concepts, the interest in ReLU DNNs picked up over the last decade as well. Various properties of ReLU gates have been the subject of extensive research including both single-layer structures  (see, e.g., \cite{PuthawalaKLDH22,FuruyaPLH23,MBBDN23,PuthawalaLDH22,Pal21,Clum22})
as well as more complex multilayered ones (see, e.g., \cite{BalMalZech19,ZavPeh21,AMZ24,Stojnictcmspnncapdiffactrdt23,Stojnictcmspnncapdinfdiffactrdt23}). Of particular interest have been the invertibility or injectivity abilities of ReLU activations as they make them a bit different from, say, more traditional \emph{sign} perceptrons or other nonlinear ones. For example, just mere existence of at least $m_1\geq m_0$ nonzero elements at the output of the first layer is sufficient to recover the generating input (a non-degenerative scenario with any subset of rows/columns  of $A^{(i)}$ being of full rank is assumed throughout the paper; in statistical contexts of our interest here, this typically happens with probability 1). While potential existence of the injectivity is relatively easy to observe, it is highly nontrivial to determine the minimal length of $\bar{\y}$ that ensures its existence (see, e.g., \cite{PuthawalaKLDH22,FuruyaPLH23,MBBDN23,PuthawalaLDH22,Pal21,Clum22,Stojnicinjrelu24}). In fact, this is already very complicated to do evem for a single layer network \cite{MBBDN23,Stojnicinjrelu24}.

\emph{Precisely} determining sequence  $m_0,m_1,m_2,\dots,m_l$ for which ReLU DNNs are injective is the main topic of this paper. As we will be working in the high-dimensional linear (proportional) regime, this  effectively translates into determining the corresponding \emph{absolute} expansion sequence, $\alpha_i,i=1,2,\dots,l$. It is not that difficult to see that in order to have an $l$-layer ReLU NN, ${\mathcal N}_{1:l}$, injective, all its subnetworks ${\mathcal N}_{1:1},{\mathcal N}_{1:2},\dots,{\mathcal N}_{1:(1-1)}$, must be injective as well. A sequence $\alpha_i,i=1,2,\dots,(l-1)$ that ensures this will be called \emph{injectively admissible}. Of particular interest are the  minimal ones, i.e., the ones that require minimal necessary expansion in each of the layers. They will be called \emph{minimally injectively admissible}. Following the usual terminology (see, e.g., \cite{PuthawalaKLDH22,FuruyaPLH23,MBBDN23,PuthawalaLDH22}), we, for an injectively admissible sequence $\alpha_i,i=1,2,\dots,(l-1)$ (for $l=1$ the sequence is empty and automatically admissible), and  statistical $A^{(i)},i=1,2,\dots,l$, formally define
 \begin{eqnarray}
\bl{\emph{\textbf{$l$-layer ReLU NN, ${\mathcal N}_{1:l}$, injectivity capacity:}}} & & \nonumber   \\
\hspace{-0in} \alpha_{ReLU}^{(inj)}\(\alpha_1,\alpha_2,\dots,\alpha_{l-1}\)  \triangleq  \min \hspace{1in} & &  \hspace{-1in} \alpha_l\nonumber \\
  \mbox{subject to} \hspace{1in} & &  \hspace{-1in} \lim_{n\rightarrow\infty}\mP_{{\mathcal A}_{1:l}} (\forall\bar{\x},\nexists \x\neq \bar{\x} \quad \mbox{such that}\quad \bar{f}_{nn}(\bar{\x};{\mathcal A}_{1:l}) = \bar{f}_{nn}(\x;{\mathcal A}_{1:l}) )=1. \nonumber \\\label{eq:ex1a3}
\end{eqnarray}
To facilitate the exposition we assume throughout the paper  that $A^{(i)}$'s are comprised of iid standard normals (as mentioned in \cite{Stojnicinjrelu24}, all of our results are easily generalizable to various other statistics that can be pushed through the Lindeberg variant of the central limit theorem). We also adopt the convention that the subscripts next to $\mP$ and $\mE$ denote the underlying source of randomness (these subscripts are left unspecified when the source of randomness is clear from the context). It is also not that difficult to establish the definition of
 \begin{eqnarray}
\bl{\emph{\textbf{\underline{Minimally} injectively admissible sequence:}}} \quad\quad  \alpha_i = \alpha_{ReLU}^{(inj)}\(\alpha_1,\alpha_2,\dots,\alpha_{i-1}\),i=1,2,\dots,l.\label{eq:ex1a3a0}
\end{eqnarray}
All our results will be obtained for generally injectively admissible sequences. To obtain concrete numerical values of the capacities we will then consider the minimally injectively admissible ones as they provide architecture with the minimal expansion in each of the layers.

\subsection{Related prior work}
\label{sec:relwork}

\noindent $\star$ \underline{\emph{Deep learning compressed sensing}:} A particularly prominent role of injectivity over the last several years appeared  in deep learning approaches to compressed sensing (or general structured objects recovery). In \cite{BoraJPD17} a deep learning compressed sensing paradigm was put forth where properly trained deep ReLU nets are used as sparse signals generative engines. Nice ReLU analytical properties allowed for utilization of gradient methods to recover both generating input and the corresponding network output (i.e., the desired sparse signal). Experimental results showed a superior performance compared to standard convexity based techniques with a level of compression decreased by 4-5 times compared to say LASSO. Analytical results showed dimensional orders and errors that match the best known ones of the convex methods provided that gradient solves the underlying optimization. \cite{HHHV18a} then showed a polynomial running of the gradient provided a logarithmic layers expansions. On the other hand, \cite{DaskalakisRZ20a}) showed that constant expansion in principle suffices for compressed sensing  while the results of  \cite{LeiJDD19} (when taken together with  \cite{ShahH18}) achieved the same constant type of expansion with a particularly tailored layer-wise Gaussian elimination algorithm. Many other utilizations of generative models, or generative adversarial networks (GAN), appeared in further improving  compressed sensing, phase retrieval or denoising   \cite{DharGE18,HandLV18,MardaniSDPMVP18,HHHV18}. As noted in most of the above works, invertibility/injectivity of the generating models/networks is a key precondition that allows for application of such objects in deep learning compressed sensing. Moreover, its precise dimensional characterizations critically impact the architecture of the networks, computational complexity of the training and recovery algorithms, and the accuracy of the theoretical guarantees that rely on them. Naturally, studying the above defined injectivity capacities in full detail imminently gained traction and became extraordinarily relevant research topic on its own as well.

\noindent $\star$ \underline{\emph{Injectivity capacity versus classical perceptrons capacities}:}  Before switching to precise studies of capacities, one thing needs to be carefully addressed as well. Namely, although conceptually somewhat connected, the  capacities defined in the previous section are actually  different from the ones typically associated with classical spherical
\cite{Schlafli,Cover65,Winder,Winder61,Wendel,Cameron60,Joseph60,Gar88,Ven86,BalVen87,SchTir02,SchTir03,StojnicGardGen13,StojnicGardSphErr13,Talbook11a,Stojnicnegsphflrdt23}
or binary perceptrons \cite{FPSUZ17,Gar88,GarDer88,GutSte90,KraMez89,AbbLiSly21b,PerkXu21,AbbLiSly21a,AlwLiuSaw21,AubPerZde19,GamKizPerXu22}. There are two  key difference: \textbf{\emph{(i)}} The injectivity capacities are defined for the whole network whereas the classical ones are typically defined for a single gate; and \textbf{\emph{(ii)}} The classical perceptrons' capacities usually relate to the network (gate) ability to store/memorize patterns -- fundamentally different from ability to (uniquely) invert them. Some aspects of these differences can be bridged though. For example, regarding the first difference, \cite{EKTVZ92,BHS92,MitchDurb89,Stojnictcmspnncapliftedrdt23,Stojnictcmspnncaprdt23} show how single layer associative memories extend to multilayered ones as well.

\noindent $\star$ \underline{\emph{Connecting injectivity  and classical perceptrons capacities}:} While the second of the above differences by the definition can not be conceptually bridged, it can be bridged in mathematical terms. Namely, as shown in \cite{Stojnicinjrelu24},  a so-called  $\ell_0$ \emph{spherical perceptron} reformulation of the single layer injectivity capacities is possible so that they resemble (or become mathematically equivalent to) variants of the classical ones. Further connections with classical perceptron studies have been established while focusing on the injectivity of 1-layer ReLU NNs. In particular, \cite{PuthawalaKLDH22} combined union bounding with spherical perceptron capacity characterizations \cite{Cover65,Wendel,Winder,Winder61} to obtain $\approx 3.3$ and $\approx 10.5$  as respective lower and upper bounds on the single layer injectivity capacity under the Gaussianity of $A$ (for a further upper-bound decrease  to $\approx 9.091$ see \cite{MBBDN23} and reference therein  \cite{Pal21,Clum22}; on the other hand, \cite{PuthawalaKLDH22} showed that allowing for optimal network's weights choice lowers the capacity to $2$). Appearance of a discrepancy between the bounds strongly suggested their looseness. While opting for union bounding comes as a no surprise given the highly nonconvex nature of the underlying problems, the resulting deviation from exactness is to be expected as well. As discussed in a long line of work on various perceptrons models \cite{Tal06,Talbook11a,StojnicGardSphErr13,StojnicDiscPercp13}, utilization of classical techniques when facing nonconvex problems typically results in built-in suboptimality.

\noindent $\star$ \underline{\emph{Geometric approaches}:} Different approaches were taken in \cite{Pal21,Clum22} where a connection to high-dimensional integral geometry studies of intersecting random subspaces is established. Relying  on the kinematic formula \cite{SW08} (and somewhat resembling earlier compressed sensing approaches of  \cite{ALMT14,DonohoPol}), \cite{Pal21} considered a heuristical Euler characteristics approximation approach and obtained a slightly  better capacity estimate $\approx 8.34$.

\noindent $\star$ \underline{\emph{Probabilistic approaches}:}  Exactness of this prediction was considered in \cite{MBBDN23}. Following into the  footsteps of \cite{Pal21,Clum22} and connecting further ReLU injectivity and intersecting random subspaces, \cite{MBBDN23} showed that application of Gordon's escape through a mesh theorem \cite{Gordon85,Gordon88}  results in a much large  $\approx 23.54$ capacity bound. The authors of \cite{MBBDN23} then made use of the (plain) \emph{random duality theory} (RDT) that Stojnic created in a long line of work \cite{StojnicCSetam09,StojnicICASSP10var,StojnicGorEx10,StojnicRegRndDlt10,StojnicGardGen13} and lowered the capacity upper-bound to $\approx 7.65$.  This was both significantly better than what they got through the Gordon's theorem and also sufficiently good to refute the above mentioned  Euler characteristics approximation based prediction. \cite{MBBDN23} continued further by utilizing replica methods from statistical physics and obtained single layer injectivity capacity prediction $\approx 6.698$. All these results were closely matched by Stojnic in \cite{Stojnicinjrelu24} through the utilization of the \emph{Fully lifted random duality theory} (fl RDT).

\noindent $\star$  \underline{\emph{``Qualitative'' vs ``quantitative'' performance characterizations:}}  The above results mostly relate to single layer ReLU injectivity. Due to underlying difficulties, studying, of a presumably simpler, single layer injectivity equivalent is often undertaken as a first step on a path towards handling multilayered networks (see, e.g. \cite{Pal21,Clum22,MBBDN23,PuthawalaKLDH22,Stojnicinjrelu24}). One however needs to note a significant difference between the role studying of a single layer injectivity plays in both qualitative  and quantitative performance analyses. In qualitative analyses (that mostly focus on correct dimensional orders) results obtained for a single layer almost automatically extend to multilayered structures as the dimension orders are to a large degree preserved. In quantitative analyses things are way more different. Even when one can precisely characterize performance of one layer, trivial extensions to higher layers typically incur a substantial suboptimality (while the  dimensional orders are likely to be preserved the concrete associated constants dramatically change). For example, if a single layer injectivity capacity is $\alpha\approx 6.7$ then the corresponding $l$-layer one is for sure trivially upper-bounded by $\alpha^l\approx 6.7^l$. From the qualitative performance characterization point of view, this is sufficient to ensure constant  expansion per both layer and network as a whole which on its own is a remarkable property. On the other hand, from the quantitative performance analysis point of view, such a bound is expected to be highly suboptimal. Moreover, it grows at a much faster pace than what is observed in  practical implementations (see, e.g., \cite{BoraJPD17} where the needed expansion is $\sim 40$ which is way smaller than $6.7^3\approx 300$). As such it can be way too conservative and result in projecting much larger network architectures than needed, potentially causing computational intractabilities of both training and recovery algorithms. In other words, if significantly suboptimal, it could be more suited as a cautious intuitive guidance  rather than as a precise recipe for designing network architectures. In general, both types of analysis are useful, one just needs to be careful when and how to use them. The qualitative ones are more adapted for scenarios where one needs quick intuitive assessments whereas the quantitative ones are more tailored for fine-grained precise designs.

In what follows, we consider  multilayered deep ReLU NNs and focus on a \emph{quantitative} type of analysis that provides injectivity capacity upper-bounds much lower than the ones obtained through trivial single layer generalizations. Before proceeding  further with the technical analysis, we below summarize some of our key results.

\subsection{Our contributions}
\label{sec:contrib}

As stated above our focus is on precise studying of deep ReLU NN injectivity capacities.  A fully connected feed forward network architecture is considered in a statistical so-called \emph{linear/proportional} high-dimensional regime. This effectively means that the number of outputs of $l$-th layer, $m_l$,  is $\alpha_l$ times larger than the number of network inputs, $n=m_0$. Moreover, it remains constant as the number of inputs/outputs grows.

\begin{itemize}
  \item Following  \cite{Stojnicinjrelu24,StojnicGardGen13,StojnicGardSphErr13,Stojnicnegsphflrdt23}, we establish  a connection between studying $l$-layer ReLU NN injectivity properties and random feasibility problem (rfps). In particular (see Section \ref{sec:mathsetup}), we find
      \begin{eqnarray}      \label{eq:eqint1}
      \mbox{$l$-layer  ReLU NN injectivity} \qquad \Longleftrightarrow \qquad \mbox{$l$-extended $\ell_0$  spherical perceptron},
      \end{eqnarray}
      which then implies
      \begin{eqnarray}      \label{eq:eqint2}
      \mbox{$l$-layer  ReLU NN injectivity capacity} \qquad \Longleftrightarrow \qquad \mbox{$l$-extended $\ell_0$  spherical perceptron capacity}.
      \end{eqnarray}
   \item Relying on \emph{Random duality theory}  (RDT) we create a generic program for studying deep ReLU NNs and their injectivity properties. We introduce two notions of injectivity, \emph{weak} and \emph{strong} (see Section \ref{sec:2lay}). The weak one relates to invertibility over all possible inputs, whereas the strong one relates to invertibility of any given input. For both notions and for any number of layers $l$, we provide upper-bounds on $l$-layer ReLU NN injectivity capacity  (see Sections \ref{sec:ubrdt} and \ref{sec:3ubrdt}).
 \item   For the first three layers explicit numerical values of the capacity bounds and all associated RDT parameters are provided (see Section \ref{sec:ubrdt}  and Tables \ref{tab:tab1} and  \ref{tab:tab2} for 2-layer nets and Section \ref{sec:3ubrdt}  and Tables \ref{tab:tab3}, and \ref{tab:tab4} for 3-layer nets).
\item Table \ref{tab:tab0} previews the change in \emph{weak} injectivity capacity as the number of layers increases. A remarkable rapid onset of \emph{expansion saturation effect} is observed. Namely, already for nets with 4 layers necessary expansion per layer decreases fairly close to 1, thereby almost reaching level of no needed expansion.
\begin{table}[h]
\caption{ Deep ReLU NN \emph{weak} injectivity capacity and layers expansions ; $m_i$ -- \# of nodes in the $i$-th layer}\vspace{.1in}
\centering
\def\arraystretch{1.2}
\begin{tabular}{||c||c||c||c||c||}\hline\hline
 $l$ (\# of layers)                                     &  $\mathbf{1}$ &   $\mathbf{2}$
                                                                    &  $\mathbf{3}$       &  $\mathbf{4}$  \\ \hline\hline
    $\begin{matrix}
     \mbox{\textbf{Injectivity capacity} (upper bound)} \\
     \lp \alpha_l=\lim_{n\rightarrow\infty}\frac{m_l}{n} =\lim_{m_0\rightarrow\infty}\frac{m_l}{m_0} \rp  \end{matrix}_{\big.}^{\big.}$
                     &  $\bl{\mathbf{6.7004}}$ &   $\bl{\mathbf{8.267}}$
                                                                    &  $\bl{\mathbf{9.49}}$       &  $\bl{\mathbf{10.124}}$  \\ \hline\hline
    $\begin{matrix}
     \mbox{\textbf{Layer's expansion}} \\
  \lp \zeta_l=\lim_{n\rightarrow\infty}\frac{m_l}{m_{l-1}}=\frac{\alpha_l}{\alpha_{l-1}} \rp
  \end{matrix}_{\big.}^{\big.}$
        &  $\prp{\mathbf{6.7004}}$ &   $\prp{\mathbf{1.2338}}$
                                                                    &  $\prp{\mathbf{1.1479}}$       &  $\prp{\mathbf{1.0668}}$
      \\ \hline\hline
  \end{tabular}
\label{tab:tab0}
\end{table}
  \item To further lower the capacity upper bounds, we develop a powerful \emph{Lifted} RDT based program (see Sections \ref{sec:ubplrdt}  and \ref{sec:plllay}). For 2-layer nets we provide concrete numerical values and observe that they indeed lower the corresponding plain RDT  ones  (see Section \ref{sec:ubplrdt} and Tables \ref{tab:tab7}and \ref{tab:tab8}).
\item Implications for compressed sensing seem rather dramatic. As the reciprocal of the values given in Table \ref{tab:tab0} are tightly connected to sparsity and undersampling  ratios in compressed sensing, provided that nets are sufficiently generalizable and that the underlying recovery algorithms run fast (which practical implementations suggest is the case), the above results allow for unprecedented improvement unreachable by any of the currently practically usable non NN based compressed sensing algorithms.
\end{itemize}

\section{2-layer ReLU NN}
 \label{sec:2lay}

To make the presentation smoother we start more technical discussions by considering 2-layered nets as the simplest example of the multilayered ones. For such networks injectively admissible sequence has only one component, $\alpha_1$, which, by the earlier definition, is not smaller than the injectivity capacity of a single ReLU layer. Keeping this in mind and proceeding in a similar fashion as in \cite{Stojnicinjrelu24}, we rely on \cite{StojnicGardGen13,StojnicGardSphErr13,StojnicDiscPercp13,Stojnicbinperflrdt23,Stojnicnegsphflrdt23} and the random feasibility problems (rfps) considerations established therein to observe that for $l=2$ the condition under the probability in (\ref{eq:ex1a3a0}) is directly related to the following \emph{feasibility} optimization problem
 \begin{eqnarray}
 {\mathcal F}(\cA_{1:2},\alpha_1,\alpha_2): \qquad \qquad    \mbox{find} & & \x\nonumber \\
  \mbox{subject to} & &  A^{(1)}\x=\z\nonumber \\
  & &  A^{(2)}\max(\z,0)=\t\nonumber \\
  & & \|\max(\t,0)\|_0 < 2n. \label{eq:ex1a4}
\end{eqnarray}
To see the connection, we first note that if the above problem is infeasible then $\bar{\y}^{(2)}$ in (\ref{eq:ex1a2}) must have at least $2n$ nonzero components.  Under the assumption that $\alpha_1$ is injectively admissible (and under the non-degenerative assumption stated a bit earlier) the $\bar{\y}^{(1)}$'s number of the nonzeros  is at least $n=m_0$. This then implies that $\bar{\y}^{(2)}$ has at least $n$ nonzeros.

Assume that there are two outputs of the first layer, say $\bar{\y}^{(1,1)}$ and $\bar{\y}^{(1,2)}$ that both generate $\bar{\y}^{(2)}$. Let the sets of indices of $\bar{\y}^{(1,1)}$ and $\bar{\y}^{(1,2)}$ nonzero components be $S_1$ and $S_2$, respectively. Similarly, let the set of  indices of $\bar{\y}^{(2)}$'s nonzero components  be $S_0$. Set
 \begin{eqnarray}
 S_o  & = & S_1\cap S_2 \nonumber \\
 S_{d_1}  & = & S_1\setminus S_2 \nonumber\\
 S_{d_2}  & = & S_2\setminus S_1. \label{eq:ex1a4a0a0}
\end{eqnarray}
For a lack of injectivity,  there must exist an $\x\in \mR^n$ such that
 \begin{eqnarray}
 \bar{\y}_{S_1}^{(1,1)} & = & A_{S_1,:}^{(1)}\bar{\x} \nonumber \\
 \bar{\y}_{S_2}^{(1,2)} & = & A_{S_2,:}^{(1)} \x \nonumber \\
A_{S_0,S_1}^{(2)}A_{S_1,:}^{(1)}\bar{\x}=
A_{S_0,S_1}^{(2)}\bar{\y}_{S_1}^{(1,1)}
& = &
 A_{S_0,S_2}^{(2)}\bar{\y}_{S_2}^{(1,2)}  = A_{S_0,S_2}^{(2)}A_{S_1,:}^{(1)} \x. \label{eq:ex1a4a0}
\end{eqnarray}
The last equality in (\ref{eq:ex1a4a0}) further gives that the following condition must also be met to ensure the lack of injectivity
 \begin{eqnarray}
A_{S_0,S_1}^{(2)}A_{S_1,:}^{(1)}\bar{\x} - A_{S_0,S_2}^{(2)}A_{S_1,:}^{(1)}\x  & = &  0
\nonumber \\
\quad \Longleftrightarrow \quad\hspace{.2in}
\begin{bmatrix}
  A_{S_0,S_o}^{(2)}A_{S_o,:}^{(1)} + A_{S_0,S_{d_1}}^{(2)}A_{S_{d_1},:}^{(1)}    &
- A_{S_0,S_o}^{(2)}A_{S_o,:}^{(1)}  - A_{S_0,S_{d_2}}^{(2)}A_{S_{d_2},:}^{(1)}
\end{bmatrix}   \begin{bmatrix}
                   \bar{\x}   \\    \x
                \end{bmatrix}  &   =  & 0. \label{eq:ex1a4a0a1}
\end{eqnarray}
Under the non-degenerative assumption stated a bit earlier and keeping in mind that the cardinalities of $S_0$, $S_1$, and $S_2$ satisfy
 \begin{eqnarray}
2n \leq |S_0| \leq  \min(| S_1 |,|S_2|), \label{eq:ex1a4a0a2}
\end{eqnarray}
one has that
 \begin{eqnarray}
 \mbox{rank} \lp \begin{bmatrix}
  A_{S_0,S_o}^{(2)}A_{S_o,:}^{(1)} + A_{S_0,S_{d_1}}^{(2)}A_{S_{d_1},:}^{(1)}    &
- A_{S_0,S_o}^{(2)}A_{S_o,:}^{(1)}  - A_{S_0,S_{d_2}}^{(2)}A_{S_{d_2},:}^{(1)}
\end{bmatrix}  \rp \geq 2n, \label{eq:ex1a4a0a3}
\end{eqnarray}
which contradicts existence of $\begin{bmatrix}
                   \bar{\x}   \\    \x
                \end{bmatrix} \in\mR^{2n}$
that satisfies (\ref{eq:ex1a4a0a1}) and therefore the presumed nonexistence of injectivity. The likelihood of having non-degenerative assumption in place over continuous $A^{(1)}$ and $A^{(2)}$ that we will consider later on (namely the standard normal ones)   is trivially zero (in other words, it is a possible, but improbable event). We then say that infeasibility of (\ref{eq:ex1a4}) implies a \emph{typical} strong injectivity or  for brevity  just ``\emph{strong injectivity}''. From a practical point of view, a weaker notion might be of even greater interest. Such a notion assumes that for a given generative $\bar{\x}$ there is no $\x$ such that (\ref{eq:ex1a4a0}) holds. The non-degenerative assumption together with the above reasoning then easily gives that the infeasibility of (\ref{eq:ex1a4}) with a factor 2 removed implies the ``\emph{weak injectivity'}'.

To make writing easier we set
 \begin{eqnarray}
f_{(s)}  & \triangleq & \|\max(\t,0)\|_0 -2n  \nonumber \\
f_{(w)} & \triangleq&  \|\max(\t,0)\|_0 - n  \nonumber \\
f_{inj} & \triangleq &  \begin{cases}
                         f_{(s)} , & \mbox{for \emph{strong injectivity}} \\
                       f_{(w)}  , & \mbox{for \emph{weak injectivity}}.
                     \end{cases} \label{eq:ex1a4a0a4}
\end{eqnarray}
When needed for the concreteness  in the derivations below, we specialize to the weak case and take $f_{inj}=f_{(w)}$. Very minimal modifications of our final results will then automatically be applicable to the strong case as well. One should also note that there is no distinction between the above two notions of injectivity in 1-layer networks.

Returning back to (\ref{eq:ex1a4}) and noting the scaling invariance of the optimization therein, one can, without loss of generality, assume the unit sphere restriction of $\x$. Moreover, following further the trends set in \cite{StojnicGardGen13,StojnicGardSphErr13,StojnicDiscPercp13,Stojnicbinperflrdt23,Stojnicnegsphflrdt23}, one can also introduce  $f_a(\x,\z,\t):\mR^{m_0+m_1+m_2}\rightarrow\mR$ as an artificial objective and transform the (\emph{random}) feasibility problem  (rfp) given in  (\ref{eq:ex1a4}) into the following (\emph{random}) optimization problem (rop)
 \begin{eqnarray}
\hspace{-.8in}\bl{\textbf{\emph{2-extended $\ell_0$ spherical perceptron}}} \qquad   \qquad   \qquad    \min_{\x,\z,\t} &   & f_a(\x,\z,\t)\nonumber \\
  \mbox{subject to} & &  A^{(1)}\x=\z\nonumber \\
    & &  A^{(2)}\max(\z,0)=\t\nonumber \\
  & & f_{inj} <0  \nonumber \\
  & & \|\x\|_2=1. \label{eq:ex1a5}
\end{eqnarray}
Any optimization problem is actually solvable only if it is feasible. Under the feasibility assumption, (\ref{eq:ex1a5}) can then be rewritten as
\begin{equation}
\xi_{f}^{(2)}(f_a) = \hspace{-.02in} \min_{  \|\x\|_2=1,  f_{inj} < 0  } \max_{\y_{(i)}\in\cY_i}  \lp f_a(\x,\z,\t) +\y_{(1)}^T A^{(1)}\x +\y_{(2)}^T A^{(2)}\max(\z,0) -\y_{(1)}^T  \z -\y_{(2)}^T  \t  \rp,
 \label{eq:ex2}
\end{equation}
where $\cY_i=\mR^{m_1}$. Since $f_a(\x,\z,\t)$ is an artificial object, one can specialize back to $f_a(\x,\z,\t)=0$ and write
\begin{eqnarray}
\xi_{f}^{(2)}(0) =  \min_{\|\x\|_2=1,  f_{inj} < 0 } \max_{\y_{(i)}\in\cY_i}  \lp \y_{(1)}^T A^{(1)}\x +\y_{(2)}^T A^{(2)}\max(\z,0) -\y_{(1)}^T  \z -\y_{(2)}^T  \t  \rp.
 \label{eq:ex2a1}
\end{eqnarray}
From (\ref{eq:ex2a1}), one can now easily see the main point behind the connection to rfps. Namely, the existence of a triplet $\x,\z,\t$ such that $\|\x\|_2=1$, $f_{inj} < 0$  and  $A^{(1)}\x=\z$, $A^{(2)}\max(\z,0)=\t$ i.e., such that (\ref{eq:ex1a4}) is feasible, ensures that (\ref{eq:ex2a1})'s inner maximization  can do no better than make $\xi_{f}^{(2)}(0) =0$. If such a triplet does not exist, then at least one of the equalities in $A^{(1)}\x= \z$ or $A^{(2)}\max(\z,0)= \t$ is violated which allows the inner maximization to trivially achieve $\xi_{f}^{(2)}(0) =\infty$. Since  $\xi_{f}^{(2)}(0) =\infty$ and $\xi_{f}^{(2)}(0) >0$ are no different from the feasibility point of view, the underlying optimization in (\ref{eq:ex2a1}) can be viewed as $\y_{(1))},\y_{(2))}$ scaling  invariant for all practical purposes. Such an invariance allows restriction to $\|\y_{(1)}\|_2=1$ and $\|\y_{(2)}\|_2=\frac{1}{\sqrt{n}}$  which in return guarantees boundedness of $\xi_{f}^{(2)}(0)$. Having all of this in mind, one recognizes the importance of
\begin{eqnarray}
\xi_{ReLU}^{(2)} =  \min_{\|\x\|_2=1,f_{inj} < 0} \max_{\|\y_{(1)}\|_2=1,|\y_{(2)}\|_2=\frac{1}{\sqrt{n}}}  \lp \y_{(1)}^T A^{(1)}\x +\y_{(2)}^T A^{(2)}\max(\z,0) -\y_{(1)}^T  \z -\y_{(2)}^T  \t  \rp,
 \label{eq:ex3}
\end{eqnarray}
for the analytical characterization of (\ref{eq:ex1a4}). In fact, it is the sign of $\xi_{ReLU}^{(2)}$  (i.e., of the objective in (\ref{eq:ex3}))  that determines (\ref{eq:ex1a4})'s  feasibility. If $\xi_{ReLU}^{(2)}>0$ then (\ref{eq:ex1a4}) is infeasible and the network is (typically) injective. A combination with (\ref{eq:ex1a4}) allows then for the following \emph{typical} rewriting of (\ref{eq:ex1a3})
 \begin{eqnarray}
\hspace{-0in} \alpha_{ReLU}^{(inj)}(\alpha_1) & \triangleq  &
 \max \{\alpha_2 |\hspace{.08in}
 \lim_{n\rightarrow\infty}\mP_{\cA_{1:2}} (\nexists \x\neq \bar{\x} \quad \mbox{such that}\quad \bar{f}_{nn}(\bar{\x},\cA_{1:2})=\bar{f}_{nn}(\x,\cA_{1:2}))=1\}
  \nonumber \\
  & = & \max \{\alpha_2 |\hspace{.08in}  \lim_{n\rightarrow\infty}\mP_{\cA_{1:2}}   \lp{\mathcal F}(\cA_{1:2},\alpha_1,\alpha_2) \hspace{.07in}\mbox{is feasible} \rp\longrightarrow 1\}
  \nonumber \\
  & = & \max \{\alpha_2 |\hspace{.08in}  \lim_{n\rightarrow\infty}\mP_{\cA_{1:2}}  \lp\xi_{ReLU}^{(2)}(0)>0\rp\longrightarrow 1\}.
  \label{eq:ex4}
\end{eqnarray}
To handle (\ref{eq:ex3}) (and ultimately  (\ref{eq:ex4})) we rely on Random duality theory (RDT) developed in a long series of work \cite{StojnicCSetam09,StojnicICASSP10var,StojnicCSetamBlock09,StojnicICASSP10block,StojnicRegRndDlt10}. This is shown next.

\subsection{Upper-bounding $\alpha_{ReLU}^{(inj)}(\alpha_1)$ via Random Duality Theory (RDT)}
\label{sec:ubrdt}

We  start by providing a brief summary of the main RDT principles and then continue by showing, step-by-step, how each of them relates to the problems of our interest here.

\vspace{-.0in}\begin{center}
 	\tcbset{beamer,lower separated=false, fonttitle=\bfseries, coltext=black ,
		interior style={top color=yellow!20!white, bottom color=yellow!60!white},title style={left color=black!80!purple!60!cyan, right color=yellow!80!white},
		width=(\linewidth-4pt)/4,before=,after=\hfill,fonttitle=\bfseries}
 \begin{tcolorbox}[beamer,title={\small Summary of the RDT's main principles} \cite{StojnicCSetam09,StojnicRegRndDlt10}, width=1\linewidth]
\vspace{-.15in}
{\small \begin{eqnarray*}
 \begin{array}{ll}
\hspace{-.19in} \mbox{1) \emph{Finding underlying optimization algebraic representation}}
 & \hspace{-.0in} \mbox{2) \emph{Determining the random dual}} \\
\hspace{-.19in} \mbox{3) \emph{Handling the random dual}} &
 \hspace{-.0in} \mbox{4) \emph{Double-checking strong random duality.}}
 \end{array}
  \end{eqnarray*}}
\vspace{-.2in}
 \end{tcolorbox}
\end{center}\vspace{-.0in}

To make the presentation neater, we formalize all key results (simple and more complicated ones) as lemmas or theorems.

\vspace{.1in}

\noindent \underline{1) \textbf{\emph{Algebraic injectivity representation:}}}  The above considerations established a convenient connection between injectivity capacity and feasibility problems. The main points are summarized in the following lemma.
\begin{lemma} Consider a sequence of positive integers, $n=m_0,m_1,m_2$, high-dimensional linear regime, corresponding expansion coefficients $\alpha_1,\alpha_2$, and assume that $\alpha_1$ is injectively admissible. Assume a 2-layer ReLU NN with architecture $\cA_{1:2}=[A^{(1)},A^{(2)}]$ (the rows of matrix $A^{(i)}\in\mR^{m_i\times m_{i-1}},i=1,2$, being the weights of the nodes in the $i$-th layer). The network is typically injective if
\begin{equation}\label{eq:ta10}
  f_{rp}(\cA_{1:2})>0,
\end{equation}
where
\begin{equation}\label{eq:ta11}
f_{rp}(\cA_{1:2})\triangleq \frac{1}{\sqrt{n}}  \min_{\|\x\|_2=1,f_{inj} < 0} \max_{\|\y_{(1)}\|_2=1,|\y_{(2)}\|_2=\frac{1}{\sqrt{n}}}  \lp \y_{(1)}^T A^{(1)}\x +\y_{(2)}^T A^{(2)}\max(\z,0) -\y_{(1)}^T  \z -\y_{(2)}^T  \t  \rp,
\end{equation}
and
 \begin{eqnarray}
f_{(s)}  & \triangleq & \|\max(\t,0)\|_0 -2n  \nonumber \\
f_{(w)} & \triangleq&  \|\max(\t,0)\|_0 - n  \nonumber \\
f_{inj} & \triangleq &  \begin{cases}
                         f_{(s)} , & \mbox{for strong injectivity} \\
                       f_{(w)}  , & \mbox{for weak injectivity}.
                     \end{cases} \label{eq:ta11a0}
\end{eqnarray}
   \label{lemma:lemma1}
\end{lemma}
\begin{proof}
Follows immediately from the discussion presented in the previous section.
\end{proof}

\vspace{.1in}
\noindent \underline{2) \textbf{\emph{Determining the random dual:}}} As is typical within the RDT, the so-called concentration of measure property is utilized as well. This basically means that for any fixed $\epsilon >0$,  we have (see, e.g. \cite{StojnicCSetam09,StojnicRegRndDlt10,StojnicICASSP10var})
\begin{equation*}
\lim_{n\rightarrow\infty}\mP_{\cA_{1:2}}\left (\frac{|f_{rp}(\cA_{1:2})-\mE_{\cA_{1:2}}(f_{rp}(\cA_{1:2}))|}{\mE_{\cA_{1:2}}(f_{rp}(\cA_{1:2}))}>\epsilon\right )\longrightarrow 0.\label{eq:ta15}
\end{equation*}
 The following, so-called random dual theorem, is a key ingredient of the RDT machinery.
\begin{theorem} Assume the setup of Lemma \ref{lemma:lemma1}. Let the elements of $A^{(i)}\in\mR^{m_i\times m_{i-1}}$, $\g^{(i)}\in\mR^{m_i\times 1}$, and  $\h^{(i)}\in\mR^{m_{i-1}\times 1}$  be iid standard normals. Set
\vspace{-.0in}
\begin{eqnarray}
\cG_{(2)} & \triangleq & [\g^{(1)},\g^{(2)},\h^{(1)},\h^{(2)}]  \nonumber \\
\phi(\x,\z,\t,\y_{(1)},\y_{(2)}) & \triangleq &
 \lp \y_{(1)}^T \g^{(1)} + \x^T\h^{(1)}
+ \|\max(\z,0)\|_2\y_{(2)}^T  \g^{(2)} + \frac{1}{\sqrt{n}}\max(\z,0)^T\h^{(2)}   -\y_{(1)}^T  \z -\y_{(2)}^T  \t  \rp
\nonumber \\
 f_{rd}(\cG_{(2)}) & \triangleq &
\frac{1}{\sqrt{n}}  \min_{\|\x\|_2=1,f_{inj}<0} \max_{\|\y_{(1)}\|_2=1,|\y_{(2)}\|_2=\frac{1}{\sqrt{n}}}  \phi(\x,\z,\t,\y_{(1)},\y_{(2)})
  \nonumber \\
 \phi_0 & \triangleq & \lim_{n\rightarrow\infty} \mE_{\cG_{(2)}} f_{rd}(\cG_{(2)}).\label{eq:ta16}
\vspace{-.0in}\end{eqnarray}
One then has \vspace{-.02in}
\begin{eqnarray}
\hspace{-.3in}(\phi_0  > 0)   &  \Longrightarrow  & \lp \lim_{n\rightarrow\infty}\mP_{\cG_{(2)}}(f_{rd}(\cG_{(2)}) >0)\longrightarrow 1\rp
\quad  \Longrightarrow \quad \lp \lim_{n\rightarrow\infty}\mP_{ \cA_{1:2} } (f_{rp} ( \cA_{1:2} )   >0)\longrightarrow 1 \rp  \nonumber \\
& \Longrightarrow & \lp \lim_{n\rightarrow\infty}\mP_{\cA_{1:2}} \lp \mbox{2-layer ReLU NN with architecture $\cA_{1:2}$ is typically injective} \rp \longrightarrow 1\rp.\label{eq:ta17}
\end{eqnarray}
The injectivity is strong for $f_{inj}=f_{(s)}$ and weak for $f_{inj}=f_{(w)}$.
\label{thm:thm1}
\end{theorem}\vspace{-.17in}
\begin{proof}
Follows immediately as a direct $2$-fold application of the Gordon's probabilistic comparison theorem (see, e.g., Theorem B in \cite{Gordon88}). Gordon's theorem is a special case of the results obtained in \cite{Stojnicgscomp16,Stojnicgscompyx16} (see Theorem 1, Corollary 1, and Section 2.7.2 in \cite{Stojnicgscomp16} as well as Theorem 1, Corollary 1, and Section 2.3.2 in \cite{Stojnicgscompyx16}).

In particular,  term  $ \lp \y_{(1)}^T \g^{(1)} + \x^T\h^{(1)} -\y_{(1)}^T  \z   \rp$ corresponds to the lower-bounding side of the Gordon's inequality related to $A^{(1)}$. On the other hand, term  $\lp   \|\max(\z,0)\|_2\y_{(2)}^T  \g^{(2)} + \frac{1}{\sqrt{n}}\max(\z,0)^T\h^{(2)}   -\y_{(2)}^T  \t  \rp$ corresponds to the lower-bounding side related to $A^{(2)}$. Given that (\ref{eq:ta11}) contains the summation of the corresponding two terms from the other side of the inequality, the proof is completed.
\end{proof}

\vspace{.1in}
\noindent \underline{2) \textbf{\emph{Handling the random dual:}}} To handle the above random dual we follow the methodologies invented and presented in a series of papers \cite{StojnicCSetam09,StojnicICASSP10var,StojnicCSetamBlock09,StojnicICASSP10block,StojnicRegRndDlt10}. After solving the optimizations over $\x$ and $\y_{(i)}$, we find from (\ref{eq:ta16})
\begin{equation}
  f_{rd}(\cG_{(2)})  =
\frac{1}{\sqrt{n}}  \min_{ f_{inj} < 0 }
  \lp \|\g^{(1)}-\z\|_2  -\|\h^{(1)}\|_2
+  \frac{1}{\sqrt{n}}\| \|\max(\z,0)\|_2 \g^{(2)} -\t\|_2 + \frac{1}{\sqrt{n}}\max(\z,0)^T\h^{(2)}    \rp.
\label{eq:hrd1}
 \end{equation}
The above can be rewritten as
\begin{eqnarray}
  f_{rd}(\cG_{(2)})   =
\frac{1}{\sqrt{n}}  \min_{r,\z,\t} & &
  \lp \|\g^{(1)}-\z\|_2  -\|\h^{(1)}\|_2
+  \frac{1}{\sqrt{n}} \| r \g^{(2)} -\t\|_2 +  \frac{1}{\sqrt{n}}\max(\z,0)^T\h^{(2)}    \rp
\nonumber \\
\mbox{subject to} & &f_{inj}<0 \nonumber \\
& &  \|\max(\z,0)\|_2=r.
\label{eq:hrd2}
 \end{eqnarray}
 As mentioned earlier, taking for concreteness $f_{inj}=f_{(w)}$ and
writing the Lagrangian we further have
\begin{eqnarray}
  f_{rd}(\cG_{(2)})   =
\frac{1}{\sqrt{n}}  \min_{r,\z,\t} \max_{\nu,\gamma}\cL(\nu,\gamma),
\label{eq:hrd3}
 \end{eqnarray}
where
\begin{eqnarray}
\cL(\nu,\gamma) &  = &  \lp \|\g^{(1)}-\z\|_2  -\|\h^{(1)}\|_2
+   \frac{1}{\sqrt{n}}\|  r \g^{(2)} -\t\|_2 +  \frac{1}{\sqrt{n}}\max(\z,0)^T\h^{(2)}    \rp
\nonumber \\
& & +
\nu \|\max(\t,0)\|_0 - \nu n  + \gamma \|\max(\z,0)\|_2^2 -\gamma r^2.
\label{eq:hrd4}
 \end{eqnarray}
One then utilizes the \emph{square root trick} (introduced on numerous occasions in \cite{StojnicLiftStrSec13,StojnicGardSphErr13,StojnicGardSphNeg13}) to obtain
\begin{eqnarray}
\cL(\nu,\gamma)\hspace{-.07in} &  = & \hspace{-.07in}
\min_{\bar{\gamma}_1,\bar{\gamma}_2}
 \lp  \bar{\gamma}_1 +   \frac{\|\g^{(1)}-\z\|_2^2}{4 \bar{\gamma}_1}  -\|\h^{(1)}\|_2
+   \bar{\gamma}_2 +  \frac{1}{n}\frac{\| r \g^{(2)} -\t\|_2^2}{4 \bar{\gamma}_2} +  \frac{1}{\sqrt{n}}\max(\z,0)^T\h^{(2)}    \rp \nonumber \\
& & +
\nu \|\max(\t,0)\|_0 - \nu n  + \gamma \|\max(\t,0)\|_2^2 -\gamma r^2
\nonumber \\
\hspace{-.07in}  &  = & \hspace{-.07in}
\min_{\bar{\gamma}_1,\bar{\gamma}_2}
 \lp  \bar{\gamma}_1 +  \sum_{i=1}^{m_1} \frac{\lp \g_i^{(1)}-\z_i\rp^2}{4 \bar{\gamma}_1}  - \|\h^{(1)}\|_2
+   \bar{\gamma}_2 +  \frac{1}{n}\sum_{i=1}^{m_2}\frac{ \lp  r \g_i^{(2)} -\t_i\rp^2}{4 \bar{\gamma}_2}
+  \frac{1}{\sqrt{n}}\sum_{i=1}^{m_1}\max(\z_i,0)^T\h_i^{(2)}    \rp \nonumber \\
\hspace{-.07in} & & \hspace{-.07in}  +
\frac{\nu}{2} \sum_{i=1}^{m_2} (1+\mbox{sign}(\t_i)) - \nu n + \gamma \sum_{i=1}^{m_1}\max(\z_i,0)^2 -\gamma r^2.
\label{eq:hrd5}
 \end{eqnarray}
After appropriate scaling $\bar{\gamma}_i\rightarrow\bar{\gamma}_i\sqrt{n}$, $\gamma\rightarrow\frac{\gamma}{\sqrt{n}}$, $r\rightarrow r \sqrt{n}$, and cosmetic change $\frac{\nu}{2}\rightarrow \frac{\nu}{\sqrt{n}}$, concentrations, statistical identicalness over $i$, and a combination of (\ref{eq:hrd3}) and (\ref{eq:hrd5}) give
\begin{eqnarray}
 \phi_0 & \triangleq & \lim_{n\rightarrow\infty} \mE_{\cG_{(2)}} f_{rd}(\cG_{(2)})
=
\mE_{\cG_{(2)}}   \min_{r,\bar{\gamma}_1,\bar{\gamma}_2,\z_i,\t_i} \max_{\nu,\gamma}\cL_1(\nu,\gamma),
\label{eq:hrd6}
 \end{eqnarray}
where
\begin{eqnarray}
\cL_1(\nu,\gamma)
  &  = &
  \bar{\gamma}_1 +  \alpha_1 \lp \frac{\lp \g_i^{(1)}-\z_i\rp^2}{4 \bar{\gamma}_1}  + \max(\z_i,0)^T\h_i^{(2)}   +\gamma \max(\z_i,0)^2 \rp
    - 1
    \nonumber \\
    & &
+   \bar{\gamma}_2 + \alpha_2 \lp \frac{r^2\lp  \lp \g_i^{(2)} -\frac{\t_i}{r} \rp^2 + \frac{4\bar{\gamma}_2\nu}{r^2} \mbox{sign}\lp \frac{\t_i}{r}  \rp \rp  }{4 \bar{\gamma}_2}   \rp
 -\nu(2-\alpha_2)-\gamma r^2  .
\label{eq:hrd7}
 \end{eqnarray}
The Lagrangian duality also gives
\begin{eqnarray}
 \phi_0 & \triangleq & \lim_{n\rightarrow\infty} \mE_{\cG_{(2)}} f_{rd}(\cG_{(2)})
=
\mE_{\cG_{(2)}}   \min_{r,\bar{\gamma}_1,\bar{\gamma}_2,\z_i,\t_i} \max_{\nu,\gamma}\cL_1(\nu,\gamma)
\geq
\mE_{\cG_{(2)}}   \min_{r,\bar{\gamma}_1,\bar{\gamma}_2} \max_{\nu,\gamma}  \min_{\z_i,\t_i}    \cL_1(\nu,\gamma).
\label{eq:hrd8}
 \end{eqnarray}
After setting
\begin{eqnarray}
 f_{q,1} & \triangleq  &
  \mE_{\cG_{(2)} } \max_{\z_i}
 \lp \frac{\lp \g_i^{(1)}-\z_i\rp^2}{4 \bar{\gamma}_1}  + \max(\z_i,0)^T\h_i^{(2)}   +\gamma \max(\z_i,0)^2 \rp
 \nonumber \\
 f_{q,2} & \triangleq & \mE_{\cG_{(2)} } \max_{\t_i} \lp \lp \g_i^{(2)} -\frac{\t_i}{r} \rp^2 + \frac{4\bar{\gamma}_2\nu}{r^2} \mbox{sign}\lp \frac{\t_i}{r}  \rp \rp,
 \label{eq:hrd9}
 \end{eqnarray}
one notes that a quantity structurally identical to $f_{q,2}$ was already handled in \cite{Stojnicinjrelu24}. Namely, after a change of variables
\begin{eqnarray}
\nu_1 \rightarrow \frac{4\bar{\gamma}_2\nu}{r^2},
 \label{eq:hrd10}
 \end{eqnarray}
one can follow  \cite{Stojnicinjrelu24}, set
 \begin{eqnarray}\label{eq:negprac19a1a0}
\bar{a} &  =  & \sqrt{2\nu_1},
  \end{eqnarray}
and
\begin{eqnarray}\label{eq:negprac19a1a1}
\bar{f}_x & = & -\lp\frac{e^{-\bar{a}.^2/2} \bar{a}}{\sqrt{2\pi}} + \frac{1}{2}\erfc\lp \frac{\bar{a}}{\sqrt{2}}\rp   \rp \nonumber \\
\bar{f}_{21} & = & -\frac{1}{2}-\frac{\nu_1}{2}\nonumber \\
\bar{f}_{22} & = & \bar{f}_x + \frac{\nu_1}{2} \erfc\lp \frac{ \sqrt{2\nu_1}}{\sqrt{2}}  \rp\nonumber \\
\bar{f}_{23} & = & -\nu_1\lp\frac{1}{2}-\frac{1}{2} \erfc\lp \frac{ \sqrt{2\nu_1}}{\sqrt{2}}  \rp \rp \nonumber \\
\bar{f}_{2} & = & \bar{f}_{21} +\bar{f}_{22} + \bar{f}_{23},
 \end{eqnarray}
and after solving the integrals obtain
\begin{eqnarray}\label{eq:hrd10a0}
f_{q,2}=(1+\bar{f}_{2}  ).
  \end{eqnarray}
On the other hand, after setting
\begin{eqnarray}
\bar{f}_{q,1} \triangleq \max_{\z_i}
 \lp \frac{\lp \g_i^{(1)}-\z_i\rp^2}{4 \bar{\gamma}_1}  + \max(\z_i,0)^T\h_i^{(2)}   +\gamma \max(\z_i,0)^2 \rp
 \label{eq:hrd11}
 \end{eqnarray}
 and solving the optimization over $\z_i$ one finds
\begin{eqnarray}
\bar{f}_{q,1} =
\begin{cases}
\min \lp 0,  \frac{\lp\g_i^{(1)}\rp^2}{4\bar{\gamma}_1} -\frac{(\max(\g_i^{(1)}-2\h_i^{(2)}\bar{\gamma}_1,0)).^2}{4\bar{\gamma}_1(1+4\gamma\bar{\gamma}_1)}  \rp , & \mbox{if } \g_i^{(1)}\leq 0 \\
   \frac{\lp\g_i^{(1)}\rp^2}{4\bar{\gamma}_1} -\frac{(\max(\g_i^{(1)}-2\h_i^{(2)}\bar{\gamma}_1,0)).^2}{4\bar{\gamma}_1(1+4\gamma\bar{\gamma}_1)} , & \mbox{otherwise}.
\end{cases} \label{eq:hrd12}
 \end{eqnarray}
For $\g_i^{(1)} > 0$  one sets
\begin{eqnarray}
        \bar{A} & = & \frac{ \g_I^{(1)}    } { 2\bar{\gamma}_1} \nonumber \\
         \bar{B}&  = &   \frac{ \bar{\gamma}_1}  {   (1+4\gamma \bar{\gamma}_1) }  \nonumber \\
         \bar{C} & = & \bar{A} \nonumber \\
        I_{11} & = & \bar{B}\lp  \frac{1}{2} (\bar{A}^2 + 1)  \lp \erf\lp  \frac{\bar{C}}{\sqrt{2}}  \rp  + 1 \rp +
        \frac{    e^{ -\frac{\bar{C}.^2}{2} }  (2 \bar{A} - \bar{C})      }  {\sqrt{2\pi}}
         \rp  \nonumber \\
        \hat{f}_{q,1} & = &     \frac{ (\g_i^{(1)}   ).^2}  {4\bar{\gamma}_1} - I_{11}.
  \label{eq:hrd13}
 \end{eqnarray}
On the other hand, for $\g_i^{(1)}\leq 0$  one sets
\begin{eqnarray}
        \bar{A} & = & \frac{ \g_I^{(1)}    } { 2\bar{\gamma}_1} \nonumber \\
         \bar{B}&  = &   \frac{ \bar{\gamma}_1}  {   (1+4\gamma \bar{\gamma}_1) }  \nonumber \\
         \bar{C}& =  & \lp   \g_I^{(1)}    -  \sqrt{  |  -(\g_I^{(1)}    )^2 |  (1+4\gamma\bar{\gamma}_1) } \rp   \frac{1}{2\bar{\gamma}_1}
  \nonumber \\
        I_{11} & = & \bar{B}\lp  \frac{1}{2} (\bar{A}^2 + 1)  \lp \erf\lp  \frac{\bar{C}}{\sqrt{2}}  \rp  + 1 \rp +
        \frac{    e^{ -\frac{\bar{C}.^2}{2} }  (2 \bar{A} - \bar{C})      }  {\sqrt{2\pi}}
         \rp  \nonumber \\
        \hat{f}_{q,2} & = &    \frac{ (\g_i^{(1)}   )^2}  {4\bar{\gamma}_1} - I_{11}.
  \label{eq:hrd14}
 \end{eqnarray}
After solving the remaining integrals one then obtains
\begin{eqnarray}
f_{q,1}  & = &  \mE_{\cG_{(2)}}\bar{f}_{q,1}
\nonumber \\
 & = &  \int_{\g_i^{(1)},\h_i^{(2)}   }
\bar{f}_{q,1} \frac{e^{-\frac{ \lp \g_i^{(1)} \rp^2  +  \lp \h_i^{(2)} \rp^2   } {2}  }}{\sqrt{2\pi}^2} d\g_i^{(1)}d\h_i^{(2)}
\nonumber \\
& = &
\int_{\g_i^{(1)}>0 }
\hat{f}_{q,1} \frac{e^{-\frac{ \lp \g_i^{(1)} \rp^2     } {2}  }}{\sqrt{2\pi}} d\g_i^{(1)}
  +
\int_{\g_i^{(1)}\leq 0 }
\hat{f}_{q,1} \frac{e^{-\frac{ \lp \g_i^{(1)} \rp^2         } {2}  }}{\sqrt{2\pi}} d\g_i^{(1)}.
  \label{eq:hrd15}
 \end{eqnarray}
A combination of (\ref{eq:hrd7}), (\ref{eq:hrd8}), (\ref{eq:hrd9}), (\ref{eq:hrd10a0}), and (\ref{eq:hrd15}) then gives
\begin{eqnarray}
 \phi_0
\geq
\mE_{\cG_{(2)}}   \min_{r,\bar{\gamma}_1,\bar{\gamma}_2} \max_{\nu,\gamma}  \min_{\z_i,\t_i}    \cL_1(\nu,\gamma)
   =
  \min_{r,\bar{\gamma}_1,\bar{\gamma}_2} \max_{\nu,\gamma}
\lp  \bar{\gamma}_1 +  \alpha_1 f_{q,1}
    -  \gamma r^2
 +   \bar{\gamma}_2 + \alpha_2 \frac { r^2 f_{q,2}  }  {4\bar{\gamma}_2}
 -\nu(2-\alpha_2) - 1 \rp ,
\label{eq:hrd16}
 \end{eqnarray}
where $f_{q,1}$ and $f_{q,2}$ are given in  (\ref{eq:hrd15}) and  (\ref{eq:hrd10a0}), respectively. An upper bound on the injectivity capacity is then obtained for $\alpha_2$ such that $\phi_0=0$. Numerical evaluations produce the concrete parameters values given in Table \ref{tab:tab1}.
\begin{table}[h]
\caption{ RDT parameters; 2-layer ReLU NN \emph{weak} injectivity capacity ; $n\rightarrow\infty$; }\vspace{.1in}
\centering
\def\arraystretch{1.2}
\begin{tabular}{||l||c||c||c|c||c||c||c||}\hline\hline
  RDT parameters                               &  $\alpha_1$ &   $r$     & $\bar{\gamma}_1$  & $\bar{\gamma}_2$ &  $\nu$ & $\gamma$    & $\alpha_{ReLU}^{(inj)}(\alpha_1)$  \\ \hline\hline
RDT parameters values                                      & $6.7004$  & $1.7697$   & $0.8935$  & $0.9642$ & $0.5560$  & $0.3078$   & \bl{$\mathbf{8.267}$}
  \\ \hline\hline
  \end{tabular}
\label{tab:tab1}
\end{table}
One now observes that the expansion of the second layer, $8.267/6.7004\approx 1.2338$, is much smaller than of the first one, $6.7004/1=6.7004$ .

The above weak injectivity capacity can easily be complemented by the corresponding strong one. The only difference is that instead of characterization in (\ref{eq:hrd16}) we now have
\begin{eqnarray}
 \phi_0
\geq
\mE_{\cG_{(2)}}   \min_{r,\bar{\gamma}_1,\bar{\gamma}_2} \max_{\nu,\gamma}  \min_{\z_i,\t_i}    \cL_1(\nu,\gamma)
   =
  \min_{r,\bar{\gamma}_1,\bar{\gamma}_2} \max_{\nu,\gamma}
\lp  \bar{\gamma}_1 +  \alpha_1 f_{q,1}
    -  \gamma r^2
 +   \bar{\gamma}_2 + \alpha_2 \frac { r^2 f_{q,2}  }  {4\bar{\gamma}_2}
 -\nu(4-\alpha_2) - 1 \rp ,
\label{eq:hrd17}
 \end{eqnarray}
After solving the optimization in (\ref{eq:hrd17}), we obtain the results shown in Table \ref{tab:tab2}.
\begin{table}[h]
\caption{ RDT parameters; 2-layer ReLU NN \emph{strong} injectivity capacity ; $n\rightarrow\infty$; }\vspace{.1in}
\centering
\def\arraystretch{1.2}
\begin{tabular}{||l||c||c||c|c||c||c||c||}\hline\hline
  RDT parameters                               &  $\alpha_1$ &   $r$     & $\bar{\gamma}_1$  & $\bar{\gamma}_2$ &  $\nu$ & $\gamma$    & $\alpha_{ReLU}^{(inj)}(\alpha_1)$  \\ \hline\hline
RDT parameters values                                      & $6.7004$  & $1.7708$   & $ 0.9647$  & $0.8938$ & $0.3954$  & $0.3077$   & \bl{$\mathbf{12.35}$}
  \\ \hline\hline
  \end{tabular}
\label{tab:tab2}
\end{table}
The expansion of the second layer, $12.35/6.7004\approx 1.8432$, is now larger than in the weak case but still much smaller than of the first layer.

  \vspace{.1in}
\noindent \underline{4) \textbf{\emph{Double checking the strong random duality:}}}  The last step of the RDT machinery assumes double checking the strong random duality. As the underlying problems do not allow for a deterministic strong duality the corresponding reversal considerations from \cite{StojnicRegRndDlt10} are not applicable and the strong random duality is not in place. This effectively implies that the presented results are strict injectivity capacity upper bounds.

\section{Multi-layer ReLU NN}
 \label{sec:multlay}

In this section we show how to translate the above results for 2-layer NNs to the corresponding ones related to multi-layer NNs. We start with 3-layer NNs and once we establish such a translation, the move to any number of layers, $l$, will be automatic.

\subsection{3-layer ReLU NN}
 \label{sec:3lay}

 To facilitate the ensuing presentation we try to parallel the derivations from Section \ref{sec:2lay}. At the same time, we proceed in  a much faster fashion avoiding unnecessary repetitions and instead prioritizing  showing key differences. For 3-layer ReLU nets injectively admissible sequences have two components, $\alpha_1,\alpha_2$. As stated earlier, the first element of the sequence, $\alpha_1$, is not smaller than the injectivity capacity of a single ReLU layer. On the other hand the second one is not smaller than the injectivity capacity of a 2-layer net.  A generic admissible sequence $\alpha_1,\alpha_2$  is considered throughout the derivations in this section while the specializations  are deferred and made later on when we discuss concrete numerical capacity evaluations.

Having all of the above in mind, we proceed following the path traced in Section \ref{sec:2lay} and observe that for $l=3$ the condition under the probability in (\ref{eq:ex1a3a0}) is related to the following \emph{feasibility} optimization problem (basically a 3-layer analogue to (\ref{eq:ex1a4}))
 \begin{eqnarray}
 {\mathcal F}(\cA_{1:3},\alpha_1,\alpha_2,\alpha_3): \qquad \qquad    \mbox{find} & & \x\nonumber \\
  \mbox{subject to} & &  A^{(1)}\x=\z\nonumber \\
  & &  A^{(2)}\max(\z,0)=\t\nonumber \\
  & &  A^{(3)}\max(\t,0)=\t^{(1)}\nonumber \\
  & & f_{inj}<0, \label{eq:3layex1a4}
\end{eqnarray}
 where as in (\ref{eq:ex1a4a0a4})
  \begin{eqnarray}
f_{(s)}  & = & \|\max(\t^{(1)},0)\|_0 -2n  \nonumber \\
f_{(w)} & = &  \|\max(\t^{(1)},0)\|_0 - n  \nonumber \\
f_{inj} & = &  \begin{cases}
                         f_{(s)} , & \mbox{for \emph{strong injectivity}} \\
                       f_{(w)}  , & \mbox{for \emph{weak injectivity}}.
                     \end{cases} \label{eq:3layex1a4a0a4}
\end{eqnarray}
After introducing  an artificial objective  $f_a(\x,\z,\t,\t^{(1)}):\mR^{m_0+m_1+m_2+m_3}\rightarrow\mR$ and restricting $\x$ to the unit sphere, we have the following 3-layer analogue to
(\ref{eq:ex1a5})
 \begin{eqnarray}
\hspace{-.8in}\bl{\textbf{\emph{3-extended $\ell_0$ spherical perceptron}}} \qquad   \qquad   \qquad    \min_{\x,\z,\t} &   & f_a(\x,\z,\t)\nonumber \\
  \mbox{subject to} & &  A^{(1)}\x=\z\nonumber \\
    & &  A^{(2)}\max(\z,0)=\t\nonumber \\
    & &  A^{(3)}\max(\t,0)=\t^{(1)}\nonumber \\
  & & f_{inj} <0  \nonumber \\
  & & \|\x\|_2=1. \label{eq:3layex1a5}
\end{eqnarray}
Paralleling further the derivation between (\ref{eq:ex1a5}) and (\ref{eq:ex4}), we obtain
 \begin{eqnarray}
\hspace{-0in} \alpha_{ReLU}^{(inj)}(\alpha_1,\alpha_2) & \triangleq  &
 \max \{\alpha_3 |\hspace{.08in}
 \lim_{n\rightarrow\infty}\mP_{\cA_{1:3}} (\nexists \x\neq \bar{\x} \quad \mbox{such that}\quad \bar{f}_{nn}(\bar{\x},\cA_{1:3})=\bar{f}_{nn}(\x,\cA_{1:3}))=1\}
  \nonumber \\
  & = & \max \{\alpha_3 |\hspace{.08in}  \lim_{n\rightarrow\infty}\mP_{\cA_{1:3}} \lp{\mathcal F}(\cA_{1:3},\alpha_1,\alpha_2,\alpha_3) \hspace{.07in}\mbox{is feasible} \rp\longrightarrow 1\}
  \nonumber \\
  & = & \max \{\alpha_3 |\hspace{.08in}  \lim_{n\rightarrow\infty}\mP_{\cA_{1:3}}  \lp\xi_{ReLU}^{(3)}(0)>0\rp\longrightarrow 1\},
  \label{eq:3layex4}
\end{eqnarray}
where
 \begin{eqnarray}
\xi_{ReLU}^{(3)} =  \min_{\|\x\|_2=1,f_{inj} < 0} \max_{\|\y_{(1)}\|_2=1,\|\y_{(2)}\|_2=\frac{1}{\sqrt{n}},\|\y_{(3)}\|_2=\frac{1}{n}}
\phi_{rp}^{(3)},
 \label{eq:3layex3a0}
\end{eqnarray}
where
 \begin{eqnarray}
\phi_{rp}^{(3)} = \lp \y_{(1)}^T A^{(1)}\x
+\y_{(2)}^T A^{(2)}\max(\z,0)  +\y_{(3)}^T A^{(3)}\max(\t,0)
-\y_{(1)}^T  \z -\y_{(2)}^T  \t  -\y_{(3)}^T  \t^{(1)}  \rp.
 \label{eq:3layex3}
\end{eqnarray}
 Following further the trend of  Section \ref{sec:2lay} we utilize the Random duality theory (RDT) to handle (\ref{eq:3layex3a0}) and (\ref{eq:3layex3}) (and ultimately  (\ref{eq:3layex4})).

\subsubsection{Upper-bounding $\alpha_{ReLU}^{(inj)}(\alpha_1,\alpha_2)$ via Random Duality Theory (RDT)}
\label{sec:3ubrdt}

We below show how each of the four main RDT principles  is implemented. All key results are again framed as lemmas and theorems.

\vspace{.1in}

\noindent \underline{1) \textbf{\emph{Algebraic injectivity representation:}}}  We start with the following 3-layer analogue to Lemma    \ref{lemma:lemma1}.

\begin{lemma} Consider a sequence of positive integers, $m_0,m_1,m_2,m_3$ ($n=m_0$), high-dimensional linear regime, corresponding expansion coefficients $\alpha_1,\alpha_2,\alpha_3$, and assume that $\alpha_1.\alpha_2$ is an injectively admissible sequence. Assume a 3-layer ReLU NN with architecture $\cA_{1:3}=[A^{(1)},A^{(2)},A^{(3)}]$ (the rows of matrix $A^{(i)}\in\mR^{m_i\times m_{i-1}},i=1,3$, being the weights of the nodes in the $i$-th layer). The network is typically injective if
\begin{equation}\label{eq:3layta10}
  f_{rp}(\cA_{1:3})>0,
\end{equation}
where
\begin{equation}\label{eq:3layta11}
f_{rp}(\cA_{1:2}) = \frac{1}{\sqrt{n}}  \min_{\|\x\|_2=1,f_{inj} < 0} \max_{\|\y_{(1)}\|_2=1,\|\y_{(2)}\|_2=\frac{1}{\sqrt{n}},\|\y_{(3)}\|_2=\frac{1}{n}}   \phi_{rp}^{(3)},
\end{equation}
 \begin{eqnarray}
\phi_{rp}^{(3)} = \lp \y_{(1)}^T A^{(1)}\x
+\y_{(2)}^T A^{(2)}\max(\z,0)  +\y_{(3)}^T A^{(3)}\max(\t,0)
-\y_{(1)}^T  \z -\y_{(2)}^T  \t  -\y_{(3)}^T  \t^{(1)}  \rp,
 \label{eq:3layta11a0a0}
\end{eqnarray}
and
 \begin{eqnarray}
f_{(s)}  & = & \|\max(\t^{(1)},0)\|_0 -2n  \nonumber \\
f_{(w)} & = & \|\max(\t^{(1)},0)\|_0 - n  \nonumber \\
f_{inj} & = &  \begin{cases}
                         f_{(s)} , & \mbox{for strong injectivity} \\
                       f_{(w)}  , & \mbox{for weak injectivity}.
                     \end{cases} \label{eq:3layta11a0}
\end{eqnarray}
   \label{lemma:lemma1}
\end{lemma}
\begin{proof}
Follows as an automatic consequence of the discussion presented in the previous section.
\end{proof}

\vspace{.1in}
\noindent \underline{2) \textbf{\emph{Determining the random dual:}}} We again utilize the concentration of measure phenomenon and
 the following 3-layer random dual analogue of Theorem \ref{thm:thm1}.
\begin{theorem} Assume the setup of Lemma \ref{lemma:lemma1}. Let the elements of $A^{(i)}\in\mR^{m_i\times m_{i-1}}$, $\g^{(i)}\in\mR^{m_i\times 1}$, and  $\h^{(i)}\in\mR^{m_{i-1}\times 1}$  be iid standard normals. Set
\vspace{-.0in}
\begin{eqnarray}
\cG_{(3)} & = & [\g^{(1)},\g^{(2)},\g^{(3)},\h^{(1)},\h^{(2)},\h^{(3)}]  \nonumber \\
\phi_{rd}^{(3)} & = &
 \y_{(1)}^T \g^{(1)} + \x^T\h^{(1)}
+ \|\max(\z,0)\|_2\y_{(2)}^T  \g^{(2)} + \frac{1}{\sqrt{n}}\max(\z,0)^T\h^{(2)}
\nonumber \\
& &
+ \|\max(\t,0)\|_2\y_{(3)}^T  \g^{(3)} +  \frac{1}{n} \max(\t,0)^T\h^{(3)}
  -\y_{(1)}^T  \z -\y_{(2)}^T  \t - \y_{(3)}^T  \t^{(1)}
\nonumber \\
 f_{rd}(\cG_{(3)}) & =  &
\frac{1}{\sqrt{n}}  \min_{\|\x\|_2=1,f_{inj}<0} \max_{\|\y_{(1)}\|_2=1,\|\y_{(2)}\|_2=\frac{1}{\sqrt{n}},\|\y_{(3)}\|_2=\frac{1}{n}}   \phi_{rd}^{(3)}
  \nonumber \\
 \phi_0 & = & \lim_{n\rightarrow\infty} \mE_{\cG_{(3)}} f_{rd}(\cG_{(3)}).\label{eq:3layta16}
\vspace{-.0in}\end{eqnarray}
One then has \vspace{-.02in}
\begin{eqnarray}
\hspace{-.3in}(\phi_0  > 0)   &  \Longrightarrow  & \lp \lim_{n\rightarrow\infty}\mP_{\cG_{(3)}}(f_{rd}(\cG_{(3)}) >0)\longrightarrow 1\rp
\quad  \Longrightarrow \quad \lp \lim_{n\rightarrow\infty}\mP_{\cG_{(3)}} (f_{rp} (\cG_{(3)}) >0)\longrightarrow 1 \rp  \nonumber \\
& \Longrightarrow & \lp \lim_{n\rightarrow\infty}\mP_{\cA_{1:3}} \lp \mbox{3-layer ReLU NN with architecture $\cA_{1:3}$ is typically injective} \rp \longrightarrow 1\rp.\label{eq:3layta17}
\end{eqnarray}
The injectivity is strong for $f_{inj}=f_{(s)}$ and weak for  $f_{inj}=f_{(w)}$.
\label{thm:3laythm1}
\end{theorem}\vspace{-.17in}
\begin{proof}
Follows in exactly the same way as the proof of Theorem \ref{thm:thm1} through a direct $3$-fold application of the Gordon's probabilistic comparison theorem (see, e.g., Theorem B in \cite{Gordon88} as well as Theorem 1, Corollary 1, and Section 2.7.2 in \cite{Stojnicgscomp16} and Theorem 1, Corollary 1, and Section 2.3.2 in \cite{Stojnicgscompyx16}). In particular,  in addition to terms  $ ( \y_{(1)}^T \g^{(1)} + \x^T\h^{(1)} -\y_{(1)}^T  \z  )$ and $   ( \|\max(\z,0)\|_2\y_{(2)}^T  \g^{(2)} + \frac{1}{\sqrt{n}}\max(\z,0)^T\h^{(2)}   -\y_{(2)}^T  \t ) $  corresponding to the lower-bounding side related to $A^{(1)}$ and $A^{(2)}$,
the term $ (\|\max(\z,0)\|_2\y_{(2)}^T  \g^{(2)} + \frac{1}{n}\max(\z,0)^T\h^{(2)}   -\y_{(2)}^T  \t  )$ corresponds to the lower-bounding side related to $A^{(3)}$. As (\ref{eq:3layta11}) contains the summation of the corresponding three terms from the other side of the inequality, the proof is completed.
\end{proof}

\vspace{.1in}
\noindent \underline{3) \textbf{\emph{Handling the random dual:}}}  The methodologies from \cite{StojnicCSetam09,StojnicICASSP10var,StojnicCSetamBlock09,StojnicICASSP10block,StojnicRegRndDlt10} are again utilized to handle the above random dual. Optimizing over $\x$ and $\y_{(i)}$  one first  transforms the optimization in (\ref{eq:3layta16}) into
\begin{eqnarray}
  f_{rd}(\cG_{(3)}) & = &
\frac{1}{\sqrt{n}}  \min_{ f_{inj} < 0 }
  \Bigg ( \Bigg. \|\g^{(1)}-\z\|_2  -\|\h^{(1)}\|_2
+ \frac{1}{\sqrt{n}} \| \|\max(\z,0)\|_2 \g^{(2)} -\t\|_2
\nonumber \\
& &
+ \frac{1}{n} \| \|\max(\t,0)\|_2 \g^{(3)} -\t^{(1)}\|_2
 + \frac{1}{\sqrt{n}}\max(\z,0)^T\h^{(2)}
  + \frac{1}{n}\max(\t,0)^T\h^{(3)}
    \Bigg. ) \Bigg ).
\label{eq:3layhrd1}
 \end{eqnarray}
The above can then be rewritten as
\begin{eqnarray}
  f_{rd}(\cG_{(3)})   =
\frac{1}{\sqrt{n}}  \min_{r,r_2,\x,\z,\t,\t^{(1)}}
 & & \phi_{rd,1}^{(3)}
 \nonumber \\
 \mbox{subject to} & &  f_{inj}<0
 \nonumber \\
 & &  \|\max(\z,0)\|_2=r
 \nonumber \\
 & &  \|\max(\t,0)\|_2=r_2,
\label{eq:3layhrd2}
 \end{eqnarray}
 where
\begin{eqnarray}
\phi_{rd,1}^{(3)} & = &    \|\g^{(1)}-\z\|_2  -\|\h^{(1)}\|_2
+  \frac{1}{\sqrt{n}} \| r \g^{(2)} -\t\|_2
+  \frac{1}{n} \| r_2 \g^{(2)} -\t\|_2
\nonumber \\
& &
+  \frac{1}{\sqrt{n}}\max(\z,0)^T\h^{(2)}
+  \frac{1}{n}\max(\t,0)^T\h^{(3)}   .
\label{eq:3layhrd2a0}
 \end{eqnarray}
As earlier, taking for concreteness $f_{inj}=f_{(w)}$ and
writing the Lagrangian gives
\begin{eqnarray}
  f_{rd}(\cG_{(3)})   =
\frac{1}{\sqrt{n}}  \min_{r,r_2,\z,\t,\t^{(1)}} \max_{\nu,\gamma}\cL(\nu,\gamma),
\label{eq:3layhrd3}
 \end{eqnarray}
where
\begin{eqnarray}
\cL(\nu,\gamma,\gamma_2)\hspace{-.05in} &  = & \hspace{-.05in}   \|\g^{(1)}-\z\|_2  -\|\h^{(1)}\|_2
+  \frac{1}{\sqrt{n}} \| r \g^{(2)} -\t\|_2
+   \frac{1}{n}\| r_2 \g^{(2)} -\t\|_2
+  \frac{1}{\sqrt{n}}\max(\z,0)^T\h^{(2)}
 \nonumber \\
& &\hspace{-.05in}
+  \frac{1}{n}\max(\t,0)^T\h^{(3)}   +
\nu \|\max(\t^{(1)},0)\|_0 - \nu n  + \gamma \|\max(\z,0)\|_2^2 -\gamma r^2 + \gamma_2 \|\max(\t,0)\|_2^2 -\gamma_2 r_2^2.\nonumber \\
\label{eq:3layhrd4}
 \end{eqnarray}
After utilizing the \emph{square root trick} one first finds
\begin{eqnarray}
\cL(\nu,\gamma,\gamma_2) &  = &
\min_{\bar{\gamma}_1,\bar{\gamma}_2}
 \Bigg .\Bigg(  \bar{\gamma}_1 +  \frac{\|\g^{(1)}-\z\|_2^2}{4 \bar{\gamma}_1}
+   \bar{\gamma}_2 +  \frac{1}{n} \frac{\| r \g^{(2)} -\t\|_2^2}{4 \bar{\gamma}_2}
+   \bar{\gamma}_3 +  \frac{1}{n^2}\frac{\| r_2 \g^{(3)} -\t^{(1)}\|_2^2}{4 \bar{\gamma}_3}
   \Bigg.\Bigg)
\nonumber \\
& &
 -\|\h^{(1)}\|_2 +  \frac{1}{\sqrt{n}}\max(\z,0)^T\h^{(2)}  +  \frac{1}{n}\max(\t,0)^T\h^{(3)}
+
\nu \|\max(\t^{(1)},0)\|_0 - \nu n
\nonumber \\
& &  + \gamma \|\max(\z,0)\|_2^2 -\gamma r^2  + \gamma_2 \|\max(\t,0)\|_2^2 -\gamma_2 r_2^2
\nonumber \\
  &  = &
 \min_{\bar{\gamma}_1,\bar{\gamma}_2}
 \Bigg .\Bigg(  \bar{\gamma}_1 +  \sum_{i=1}^{m_1} \frac{\|\g_i^{(1)}-\z_i\|_2^2}{4 \bar{\gamma}_1}
+   \bar{\gamma}_2 +  \frac{1}{n}\sum_{i=1}^{m_2}\frac{\| r \g_i^{(2)} -\t_i\|_2^2}{4 \bar{\gamma}_2}
+   \bar{\gamma}_3 +  \frac{1}{n^2}\sum_{i=1}^{m_3}\frac{\| r_2 \g_i^{(3)} -\t_i^{(1)}\|_2^2}{4 \bar{\gamma}_3}
   \Bigg.\Bigg)
   \nonumber \\
   & &
    -\|\h^{(1)}\|_2 + \frac{1}{\sqrt{n}} \sum_{i=1}^{m_1}\max(\z_i,0)^T\h_i^{(2)}  +  \frac{1}{n}\sum_{i=1}^{m_2}\max(\t_i,0)^T\h_i^{(3)}
 \nonumber \\
& & +
\frac{\nu}{2} \sum_{i=1}^{m_2} (1+\mbox{sign}(\t^{(1)}_i)) - \nu n
+ \gamma \sum_{i=1}^{m_1}\max(\z_i,0)^2 -\gamma r^2
+ \gamma_2 \sum_{i=1}^{m_2}\max(\t_i,0)^2 -\gamma_2 r_2^2.
\label{eq:3layhrd5}
 \end{eqnarray}
Appropriate scaling $\bar{\gamma}_{1/3}\rightarrow\bar{\gamma}_{1/3}\sqrt{n}$, $\gamma\rightarrow\frac{\gamma}{\sqrt{n}}$,$\gamma_2\rightarrow\frac{\gamma_2}{n}$, $r\rightarrow r\sqrt{n}$, $r\rightarrow r n$,  and cosmetic changes $\bar{\gamma}_{2}\rightarrow \frac{\bar{\gamma}_{2}}{r}\sqrt{n}$, $\gamma_2\rightarrow \frac{\gamma_2}{r\sqrt{n}^3} $,  and  $\frac{\nu}{2}\rightarrow \frac{\nu}{\sqrt{n}}$,  concentrations and  statistical identicalness over $i$ allow one  to arrive at the following analogues of (\ref{eq:hrd7}) and \ref{eq:hrd8}
\begin{eqnarray}
\cL_1(\nu,\gamma,\gamma_2)
  &  = &
 \bar{\gamma}_1 +  \alpha_1 \lp \frac{\lp \g_i^{(1)}-\z_i\rp^2}{4 \bar{\gamma}_1}  + \max(\z_i,0)^T\h_i^{(2)}   +\gamma \max(\z_i,0)^2 \rp
    \nonumber \\
    & &
   + r\lp
      \bar{\gamma}_2 +  \alpha_2 \lp \frac{\lp \g_i^{(2)}-\frac{\t_i}{r}\rp^2}{4 \bar{\gamma}_2}  + \max\lp \frac{ \t_i }{r}  ,0\rp^T\h_i^{(3)}  +\gamma_2 \max \lp \frac{ \t_i }{r}  ,0  \rp^2 \rp
   \rp
    \nonumber \\
    & &
+   \bar{\gamma}_3 + \alpha_3 \lp \frac{r_2^2\lp  \lp \g_i^{(3)} -\frac{\t_i^{(1)}}{r_2} \rp^2 + \frac{4\bar{\gamma}_3\nu}{r_2^2} \mbox{sign}\lp \frac{\t_i^{(1)}}{r_2}  \rp \rp  }{4 \bar{\gamma}_3}   \rp
 -\nu(2-\alpha_3)-\gamma r^2  -\gamma_2 \frac{r_2^2}{r} -1, \nonumber \\
\label{eq:3layhrd7}
 \end{eqnarray}
and
\begin{equation}
 \phi_0  = \lim_{n\rightarrow\infty} \mE_{\cG_{(3)}} f_{rd}(\cG_{(3)})
=
\mE_{\cG_{(3)}}   \min_{r,\bar{\gamma}_1,\bar{\gamma}_2,\z_i,\t_i,\t^{(1)}} \max_{\nu,\gamma}\cL_1(\nu,\gamma,\gamma_2)
\geq
\mE_{\cG_{(3)}}   \min_{r,\bar{\gamma}_1,\bar{\gamma}_2,\bar{\gamma}_3} \max_{\nu,\gamma,\gamma_2}  \min_{\z_i,\t_i,\t_i^{(1)}}    \cL_1(\nu,\gamma,\gamma_2).
\label{eq:3layhrd8}
 \end{equation}
Recalling on $f_{q,1}$ and  $f_{q,2}$  from (\ref{eq:hrd9}), one recognizes that $f_{q,1}$ appears as factor multiplying $\alpha_1$ and $\alpha_2$ whereas $f_{q,2}$ appears as factor multiplying $\alpha_3\frac{r_2^2}{4\bar{\gamma}_3}$. This is then sufficient to immediately write the following analogue to (\ref{eq:hrd16})
\begin{eqnarray}
 \phi_0
& \geq &
\mE_{\cG_{(3)}}   \min_{r,r_2,\bar{\gamma}_1,\bar{\gamma}_2,\bar{\gamma}_3} \max_{\nu,\gamma,\gamma_2}  \min_{\z_i,\t_i,\t_i^{(1)}}    \cL_1(\nu,\gamma,\gamma_2)
  \nonumber \\
&  =  &
  \min_{r,r_2,\bar{\gamma}_1,\bar{\gamma}_2,\bar{\gamma}_3} \max_{\nu,\gamma,\gamma_2}
\lp  \bar{\gamma}_1 +  \alpha_1 f_{q,1}
    -  \gamma r^2
+ r\lp \bar{\gamma}_2 +  \alpha_2 f_{q,1}
    -  \gamma \frac{r_2^2}{r^2} \rp
 +   \bar{\gamma}_3 + \alpha_3 \frac { r_2^2 f_{q,2}  }  {4\bar{\gamma}_3}
 -\nu(2-\alpha_3) - 1 \rp, \nonumber \\
\label{eq:3layhrd16}
 \end{eqnarray}
where $f_{q,1}$ and $f_{q,2}$ are as in  (\ref{eq:hrd15}) and  (\ref{eq:hrd10a0}), respectively. An upper bound on the 3-layer net injectivity capacity is then obtained for $\alpha_3$ such that $\phi_0=0$.  After setting
\begin{eqnarray}
\alpha^{(0)}&  = &
\begin{bmatrix}  \alpha_2 \\
                   \alpha_2  \end{bmatrix}, \quad
r^{(0)}  = \begin{bmatrix}
                                                                      r \\
                                                                      r_2
                                                                    \end{bmatrix},
                                                                    \quad
                                                                    \bar{\gamma}^{(0)}  = \begin{bmatrix}
                                                                      \bar{\gamma}_1 \\
                                                                      \bar{\gamma}_2
                                                                    \end{bmatrix},
                                                                    \quad
\gamma^{(0)}=\begin{bmatrix}
                                                                      \gamma \\
                                                                      \gamma_2
                                                                    \end{bmatrix},
\label{eq:3layhrd16a0}
\end{eqnarray}
and conducting  numerical evaluations one obtains the concrete parameters values given in Table \ref{tab:tab3}.
\begin{table}[h]
\caption{ RDT parameters; 3-layer ReLU NN \emph{weak} injectivity capacity ; $n\rightarrow\infty$; }\vspace{.1in}
\centering
\def\arraystretch{1.2}
\begin{tabular}{||l||c||c||c|c||c||c||c||}\hline\hline
  parameters                               &  $\alpha^{(0)}$ &   $ r^{(0)}$
                                                                    & $\bar{\gamma}^{(0)}$  & $\bar{\gamma}_3$ &  $\nu$ &
                                                                    $\gamma^{(0)}$    & $\alpha_{ReLU}^{(inj)}(\alpha_1,\alpha_2)$  \\ \hline\hline
    values
& $\begin{bmatrix}
     6.7004 \\
     8.267
   \end{bmatrix}_{\big.}^{\big.}$
& $\begin{bmatrix}
     1.75 \\
     3.73
   \end{bmatrix}$
& $\begin{bmatrix}
     0.8830 \\
      1.1224
   \end{bmatrix}$
& $ 2.344125 $
& $ 1.1620$
& $\begin{bmatrix}
      0.3128 \\
   0.2952
   \end{bmatrix}$
 & \bl{$\mathbf{9.49}$}
  \\ \hline\hline
  \end{tabular}
\label{tab:tab3}
\end{table}
One observes that the expansion of the third layer, $9.49/8.267\approx 1.1479$, is even smaller than of the second one, $\approx 1.2338$ .

To complement the above weak injectivity capacity with the corresponding strong one, one utilizes, instead of characterization in (\ref{eq:3layhrd16}), the following
\begin{eqnarray}
 \phi_0
& \geq &
\mE_{\cG_{(3)}}   \min_{r,r_2,\bar{\gamma}_1,\bar{\gamma}_2,\bar{\gamma}_3} \max_{\nu,\gamma,\gamma_2}  \min_{\z_i,\t_i,\t_i^{(1)}}    \cL_1(\nu,\gamma,\gamma_2)
  \nonumber \\
&  =  &
  \min_{r,r_2,\bar{\gamma}_1,\bar{\gamma}_2,\bar{\gamma}_3} \max_{\nu,\gamma,\gamma_2}
\lp  \bar{\gamma}_1 +  \alpha_1 f_{q,1}
    -  \gamma r^2
+ r\lp \bar{\gamma}_2 +  \alpha_2 f_{q,1}
    -  \gamma \frac{r_2^2}{r^2} \rp
 +   \bar{\gamma}_3 + \alpha_3 \frac { r_2^2 f_{q,2}  }  {4\bar{\gamma}_3}
 -\nu(4-\alpha_3) - 1 \rp. \nonumber \\
\label{eq:3layhrd17}
 \end{eqnarray}
Solving the optimization in (\ref{eq:3layhrd17}) gives the numerical values shown in Table \ref{tab:tab4}.
\begin{table}[h]
\caption{ RDT parameters; 3-layer ReLU NN \emph{strong} injectivity capacity ; $n\rightarrow\infty$; }\vspace{.1in}
\centering
\def\arraystretch{1.2}
\begin{tabular}{||l||c||c||c|c||c||c||c||}\hline\hline
  parameters                               &  $\alpha^{(0)}$ &   $ r^{(0)}$
                                                                    & $\bar{\gamma}^{(0)}$  & $\bar{\gamma}_3$ &  $\nu$ &
                                                                    $\gamma^{(0)}$    & $\alpha_{ReLU}^{(inj)}(\alpha_1,\alpha_2)$  \\ \hline\hline
    values
& $\vspace{.0in}\begin{bmatrix}
     6.7004 \\
     12.35
   \end{bmatrix}_{\big.}^{\big.}$
& $\begin{bmatrix}
     1.76 \\
     7.2
   \end{bmatrix}$
& $\begin{bmatrix}
      0.8870 \\
      2.1721
   \end{bmatrix}$
& $ 5.7610$
& $ 1.5965$
& $\begin{bmatrix}
      0.3101 \\
    0.1955
   \end{bmatrix}$
 & \bl{$\mathbf{17.13}$}
  \\ \hline\hline
  \end{tabular}
\label{tab:tab4}
\end{table}
The expansion of the third layer, $17.13/12.35\approx 1.3870$, is again larger than in the weak case but still smaller than the expansion of the second layer, $1.8432$.

  \vspace{.1in}
\noindent \underline{4) \textbf{\emph{Double checking the strong random duality:}}}  As was the case for 2-layer nets, the underlying problems do not allow for a deterministic strong duality and the corresponding reversal considerations from \cite{StojnicRegRndDlt10} can not be applied which implies the absence of strong random duality and an upper-bounding nature of the above results.

\subsection{$l$-layer ReLU NN}
 \label{sec:llay}

It is now not that difficult to extend the above considerations to general  $l$-layer depth. Setting $r^{(0)}=1$, $\alpha_l^{(0)}=\alpha_l$, $\bar{\gamma}_l^{(0)}=\bar{\gamma}_l$,
 and analogously to (\ref{eq:3layhrd16a0})
 \begin{eqnarray}
\alpha^{(0)}&  = &
\begin{bmatrix}  \alpha_1 \\
                   \alpha_2 \\
                   \vdots \\
                   \alpha_{l-1} \end{bmatrix}, \quad
r^{(0)}  = \begin{bmatrix}
                                                                      r \\
                                                                      r_2 \\
                                                                      \vdots \\
                                                                      r_{l-1}
                                                                    \end{bmatrix},
                                                                    \quad
                                                                    \bar{\gamma}^{(0)}  = \begin{bmatrix}
                                                                      \bar{\gamma}_1 \\
                                                                      \bar{\gamma}_2 \\
                                                                      \vdots \\
                                                                      \bar{\gamma}_{l-1}
                                                                    \end{bmatrix},
                                                                    \quad
\gamma^{(0)}=\begin{bmatrix}
                                                                      \gamma \\
                                                                      \gamma_2 \\
                                                                      \vdots \\
                                                                      \gamma_{l-1}
                                                                    \end{bmatrix},
\label{eq:3layhrd18}
\end{eqnarray}
 one can write
 \begin{eqnarray}
 f_{q,1}^{(i)} & =  & f_{q,i} (\bar{\gamma}_i^{(0)}, \gamma_i^{(0}) ),1\leq i\leq l-1 \nonumber \\
 f_{q,2}^{(l)} & =  & f_{q,2} (\bar{\gamma}_l^{(0)}, r_{l-1}^{(0}) ),
\label{eq:3layhrd18a0}
 \end{eqnarray}
 and  analogously to (\ref{eq:3layhrd16}) and  (\ref{eq:3layhrd17})
\begin{equation}
 \phi_0
 \geq
\hspace{-.01in}   \min_{r^{(0)},\bar{\gamma}_1^{(0)}} \max_{\nu,\gamma^{(0)}}
\lp
 \sum_{i=1}^{l-1} r_{i-1}^{(0)}\lp \bar{\gamma}_{i}^{(0)} +  \alpha_{i}^{(0)} f_{q,1}^{(i)}
    -  \gamma_i^{(0)} \frac{\lp r_{i}^{(0)}\rp ^2}{\lp r_{i-1}^{(0)}\rp^2} \rp
 +   \bar{\gamma}_l^{(0)} + \alpha_l^{(0)} \frac { \lp r_{l-1} ^{(0)}\rp^2 f_{q,2}^{(l)}  }  {4\bar{\gamma}_l^{(0)}}
 -\nu(2-\alpha_l^{(0)}) - 1 \rp, \nonumber \\
\label{eq:3layhrd19}
 \end{equation}
and
\begin{equation}
 \phi_0
 \geq
\hspace{-.01in}   \min_{r^{(0)},\bar{\gamma}_1^{(0)}} \max_{\nu,\gamma^{(0)}}
\lp
 \sum_{i=1}^{l-1} r_{i-1}^{(0)}\lp \bar{\gamma}_{i}^{(0)} +  \alpha_{i}^{(0)} f_{q,1}^{(i)}
    -  \gamma_i^{(0)} \frac{\lp r_{i}^{(0)}\rp ^2}{\lp r_{i-1}^{(0)}\rp^2} \rp
 +   \bar{\gamma}_l^{(0)} + \alpha_l^{(0)} \frac { \lp r_{l-1} ^{(0)}\rp^2 f_{q,2}^{(l)}  }  {4\bar{\gamma}_l^{(0)}}
 -\nu(4-\alpha_l^{(0)}) - 1 \rp. \nonumber \\
\label{eq:3layhrd20}
 \end{equation}
After  conducting  numerical evaluations for the fourth layer one obtains the concrete parameters values given in Table \ref{tab:tab5}.
\begin{table}[h]
\caption{ RDT parameters; 4-layer ReLU NN \emph{weak} injectivity capacity ; $n\rightarrow\infty$; }\vspace{.1in}
\centering
\def\arraystretch{1.2}
\begin{tabular}{||l||c||c||c|c||c||c||c||}\hline\hline
  parameters                               &  $\alpha^{(0)}$ &   $ r^{(0)}$
                                                                    & $\bar{\gamma}^{(0)}$  & $\bar{\gamma}_3$ &  $\nu$ &
                                                                    $\gamma^{(0)}$    & $\alpha_{ReLU}^{(inj)}(\alpha_1,\alpha_2)$  \\ \hline\hline
    values
& $\begin{bmatrix}
     6.7004 \\
     8.267 \\
     9.49
   \end{bmatrix}_{\big.}^{\big.}$
& $\begin{bmatrix}
     1.73 \\
     3,68 \\
     6.7
   \end{bmatrix}$
& $\begin{bmatrix}
      0.8751 \\
      1.1205  \\
      0.9983
    \end{bmatrix}$
& $ 4.485862$
& $ 2.0769 $
& $\begin{bmatrix}
      0.3184 \\
    0.2960  \\
    0.3683
   \end{bmatrix}$
 & \bl{$\mathbf{10.124}$}
  \\ \hline\hline
  \end{tabular}
\label{tab:tab5}
\end{table}
One observes that the decreasing expansion trend continues. For the forth layer, the above results give $10.124/9.49\approx 1.0668$ which is smaller than the expansion of the third layer $ 1.1479$, Table \ref{tab:tab6} shows systematically how the weak injectivity capacity and layers expansions change as the depth of the network increases. One observes a rapid onset of an expansion saturation effect. After only 4 layers, hardly any further expansion is needed. Interestingly, a similar phenomenon is empirically observed within deep learning compressed sensing context, where after significantly expanded first layer much smaller expansions are used in higher layers (see, e.g., \cite{BoraJPD17}). Of course, overthere one he to be extra careful since, in addition to the injectivity,  the architecture must have excellent generalizabiity properties as well.
 \begin{table}[h]
\caption{ Deep ReLU NN \emph{weak} injectivity capacity and layers expansions ; $m_i$ -- \# of nodes in the $i$-th layer}\vspace{.1in}
\centering
\def\arraystretch{1.2}
\begin{tabular}{||c||c||c||c||c||}\hline\hline
 $l$ (\# of layers)                                     &  $\mathbf{1}$ &   $\mathbf{2}$
                                                                    &  $\mathbf{3}$       &  $\mathbf{4}$  \\ \hline\hline
    $\begin{matrix}
     \mbox{\textbf{Injectivity capacity} (upper bound)} \\
     \lp \alpha_l=\lim_{n\rightarrow\infty}\frac{m_l}{n} =\lim_{m_0\rightarrow\infty}\frac{m_l}{m_0} \rp  \end{matrix}_{\big.}^{\big.}$
                     &  $\bl{\mathbf{6.7004}}$ &   $\bl{\mathbf{8.267}}$
                                                                    &  $\bl{\mathbf{9.49}}$       &  $\bl{\mathbf{10.124}}$  \\ \hline\hline
    $\begin{matrix}
     \mbox{\textbf{Layer's expansion}} \\
  \lp \zeta_l=\lim_{n\rightarrow\infty}\frac{m_l}{m_{l-1}}=\frac{\alpha_l}{\alpha_{l-1}} \rp
  \end{matrix}_{\big.}^{\big.}$
        &  $\prp{\mathbf{6.7004}}$ &   $\prp{\mathbf{1.2338}}$
                                                                    &  $\prp{\mathbf{1.1479}}$       &  $\prp{\mathbf{1.0668}}$
      \\ \hline\hline
  \end{tabular}
\label{tab:tab6}
\end{table}

\section{Lifted RDT}
 \label{sec:liftrdt}

As we mentioned earlier, the strong random duality is not in place and the above capacity characterizations are not only bounds but also expected to be \emph{strict} upper bounds. In other words, one expects that they can be further lowered. The recent development of \emph{fully lifted}   (fl) RDT \cite{Stojnicsflgscompyx23,Stojnicnflgscompyx23,Stojnicflrdt23} allows to precisely evaluate by how much the above given upper bounds can be lowered. However, full implementation of the fl RDT heavily relies on a sizeable set of numerical evaluations. Since the plain RDT already requires a strong numerical effort, we find it practically beneficial to consider a bit less accurate but way more convenient  \emph{partially lifted} (pl) RDT variant \cite{StojnicLiftStrSec13,StojnicGardSphErr13,StojnicGardSphNeg13,StojnicMoreSophHopBnds10,Stojnictcmspnncapliftedrdt23}.

\subsection{Lowering upper bounds on  $\alpha_{ReLU}^{(inj)}(\alpha_1)$ via pl RDT}
\label{sec:ubplrdt}

As was the case earlier when we considered the plain RDT, we again start with 2-layer nets and later extend the results to multilayered ones. The pl RDT relies on the same principles as the plain RDT with the exception that one now deals with the  partially lifted random dual introduced in the following theorem (basically a partially lifted analogue to Theorem \ref{thm:thm1}).
\begin{theorem} Assume the setup of Theorem \ref{thm:thm1} with the elements of $A^{(i)}\in\mR^{m_i\times m_{i-1}}$, $\g^{(i)}\in\mR^{m_i\times 1}$, and  $\h^{(i)}\in\mR^{m_{i-1}\times 1}$  being iid standard normals, $c_3>0$  and
\vspace{-.0in}
\begin{eqnarray}
\cG_{(2)} & \triangleq & [\g^{(1)},\g^{(2)},\h^{(1)},\h^{(2)}]  \nonumber \\
\phi(\x,\z,\t,\y_{(1)},\y_{(2)}) & \triangleq &
 \lp \y_{(1)}^T \g^{(1)} + \x^T\h^{(1)}
+ \|\max(\z,0)\|_2\y_{(2)}^T  \g^{(2)} + \max(\z,0)^T\h^{(2)}   -\y_{(1)}^T  \z -\y_{(2)}^T  \t  \rp
\nonumber \\
 \bar{f}_{rd}(\cG_{(2)}) & \triangleq &
 \min_{r,\|\x\|_2=1,\|\max(\z,0)\|_2=r,f_{inj}<0} \max_{\|\y_i\|_2=1}  \phi(\x,\z,\t,\y_{(1)},\y_{(2)})
  \nonumber \\
 \bar{\phi}_0 & \triangleq & \min_{r>0}\lim_{n\rightarrow\infty} \frac{1}{\sqrt{n}}
 \lp
 \frac{c_3}{2} + \frac{c_3}{2n} r^2 -
\frac{1}{c_3} \log \lp \mE_{\cG_{(2)}} e^{ - c_3 \bar{f}_{rd}(\cG_{(2)}) } \rp   \rp .\label{eq:plta16}
\vspace{-.0in}\end{eqnarray}
One then has \vspace{-.02in}
\begin{eqnarray}
\hspace{-.3in}(\bar{\phi}_0  > 0)   &  \Longrightarrow  &   \lp \lim_{n\rightarrow\infty}\mP_{ \cA_{1:2} } (f_{rp} ( \cA_{1:2} )   >0)\longrightarrow 1 \rp  \nonumber \\
& \Longrightarrow & \lp \lim_{n\rightarrow\infty}\mP_{\cA_{1:2}} \lp \mbox{2-layer ReLU NN with architecture $\cA_{1:2}$ is typically injective} \rp \longrightarrow 1\rp.\label{eq:plta17}
\end{eqnarray}
The injectivity is strong for $f_{inj}=f_{(s)}$ and weak for $f_{inj}=f_{(w)}$.
\label{thm:thm2}
\end{theorem}\vspace{-.17in}
\begin{proof}
For any  fixed $r$ it follows automatically as a  $2$-fold application of Corollary 3 from  \cite{Stojnicgscompyx16} (see Section 3.2.1 and equation (86); see also Lemma 2 and equation (57) in \cite{StojnicMoreSophHopBnds10}). For example,  terms  $ \y_{(1)}^T \g^{(1)}$,  $\x^T\h^{(1)}$,  $\y_{(1)}^T  \z  $, and $\frac{c_3}{2}$ correspond to the lower-bounding side of equation (86) in \cite{Stojnicgscompyx16} related to $A^{(1)}$, whereas terms
$\|\max(\z,0)\|_2\y_{(2)}^T  \g^{(2)}$,  $\max(\z,0)^T\h^{(2)} $,   $\y_{(2)}^T  \t $, and  $\frac{c_3}{2}r^2$
 are  their $A^{(2)}$ related analogous. Given that the left hand side of (86) corresponds to $f_{rp}$ and that the minimization over $r$ accounts for the least favorable choice the proof is completed.
\end{proof}

To handle the above partially lifted random dual we rely on results from previous sections. In particular, repeating with minimal adjustments the derivations between (\ref{eq:hrd1})  and (\ref{eq:hrd5}) one arrives at
\begin{eqnarray}
 \bar{f}_{rd}(\cG_{(2)})   =
  \min_{r,\z,\t} \max_{\nu,\gamma}\cL(\nu,\gamma),
\label{eq:plhrd3}
 \end{eqnarray}
where
 \begin{eqnarray}
\cL(\nu,\gamma) \hspace{-.07in}&  = & \hspace{-.07in}
  \min_{\bar{\gamma}_1,\bar{\gamma}_2}
 \lp  \bar{\gamma}_1 +  \sum_{i=1}^{m_1} \frac{\lp \g_i^{(1)}-\z_i\rp^2}{4 \bar{\gamma}_1}  - \|\h^{(1)}\|_2
+   \bar{\gamma}_2 + \frac{1}{\sqrt{n}}  \sum_{i=1}^{m_2}\frac{ \lp  r \g_i^{(2)} -\t_i\rp^2}{4 \bar{\gamma}_2} + \frac{1}{\sqrt{n}} \sum_{i=1}^{m_1}\max(\z_i,0)^T\h_i^{(2)}    \rp \nonumber \\
& & \hspace{-.07in} +
\frac{\nu}{2} \sum_{i=1}^{m_2} (1+\mbox{sign}(\t_i)) - \nu n + \gamma \sum_{i=1}^{m_1}\max(\z_i,0)^2 -\gamma r^2
\nonumber \\
\hspace{-.07in}&  = & \hspace{-.07in}
  \min_{\bar{\gamma}_1,\bar{\gamma}_2,\gamma_{sph}}
 \Bigg. \Bigg(  \bar{\gamma}_1 +  \sum_{i=1}^{m_1} \frac{\lp \g_i^{(1)}-\z_i\rp^2}{4 \bar{\gamma}_1}
+   \bar{\gamma}_2 + \frac{1}{\sqrt{n}}  \sum_{i=1}^{m_2}\frac{ \lp  r \g_i^{(2)} -\t_i\rp^2}{4 \bar{\gamma}_2} + \frac{1}{\sqrt{n}} \sum_{i=1}^{m_1}\max(\z_i,0)^T\h_i^{(2)}     \nonumber \\
& & \hspace{-.07in} +
\frac{\nu}{2} \sum_{i=1}^{m_2} (1+\mbox{sign}(\t_i)) - \nu n + \gamma \sum_{i=1}^{m_1}\max(\z_i,0)^2 -\gamma r^2
 - \gamma_{sph} - \frac{\sum_{i=1}^{n}\lp \h_i^{(1)}\rp^2}{4\gamma_{sph}} \Bigg.\Bigg).
\label{eq:plhrd5}
 \end{eqnarray}
After appropriate scaling $c_3\rightarrow c_3\sqrt{n}$, $\bar{\gamma}_i\rightarrow\bar{\gamma}_i\sqrt{n}$, $\gamma_{sph}\rightarrow \gamma_{sph}\sqrt{n} $, $r\rightarrow r\sqrt{n}$,
$\gamma\rightarrow \frac{\gamma}{\sqrt{n}}$,  and cosmetic change $\frac{\nu}{2}\rightarrow \frac{\nu}{\sqrt{n}}$, concentrations, statistical identicalness over $i$, Lagrangian duality, and a combination of (\ref{eq:plhrd3}) and (\ref{eq:plhrd5}) give
\begin{eqnarray}
 \bar{\phi}_0  & = & \min_{r>0}\lim_{n\rightarrow\infty} \frac{1}{\sqrt{n}}
 \lp
 \frac{c_3}{2} + \frac{c_3}{2n} r^2 -
\frac{1}{c_3} \log \lp \mE_{\cG_{(2)}} e^{ - c_3 \bar{f}_{rd}(\cG_{(2)}) } \rp   \rp
\nonumber \\
& \geq &
\min_{r>0,\bar{\gamma}_1,\bar{\gamma}_2,\gamma_{sph}} \max_{\nu,\gamma}
\Bigg .\Bigg(
 \frac{c_3}{2} + \frac{c_3}{2} r^2
 +\bar{\gamma}_1
 -\frac{\alpha_1}{c_3} \log \lp \mE_{\cG_{(2)}} e^{ - c_3 \bar{f}_{q,1}  } \rp
-\gamma r^2
 +\bar{\gamma}_2 - \nu(2-\alpha_2)
 \nonumber \\
 & &
-\frac{\alpha_2}{c_3} \log \lp \mE_{\cG_{(2)}} e^{ - c_3 \bar{f}_{q,2}  } \rp
- \gamma_{sph}  - \frac{1}{c_3} \log \lp \mE_{\cG_{(2)}} e^{ c_3 \frac{\lp \h_i^{(1)}\rp^2}{4\gamma_{sph}} }\rp
\Bigg.\Bigg),
\label{eq:plhrd6}
 \end{eqnarray}
where
 \begin{eqnarray}
 \bar{f}_{q,1} & = &
  \max_{\z_i}
 \lp \frac{\lp \g_i^{(1)}-\z_i\rp^2}{4 \bar{\gamma}_1}  + \max(\z_i,0)^T\h_i^{(2)}   +\gamma \max(\z_i,0)^2 \rp
 \nonumber \\
 \bar{f}_{q,2} & = &   \max_{\t_i} \lp \lp \g_i^{(2)} -\frac{\t_i}{r} \rp^2 + \frac{4\bar{\gamma}_2\nu}{r^2} \mbox{sign}\lp \frac{\t_i}{r}  \rp \rp.
 \label{eq:plhrd9}
 \end{eqnarray}
Changing  variables
\begin{eqnarray}
\nu_1 \rightarrow \frac{4\bar{\gamma}_2\nu}{r^2},
 \label{eq:plhrd10}
 \end{eqnarray}
and recalling on  (\ref{eq:hrd12}) and equation (44) in \cite{Stojnicinjrelu24}, we further write
\begin{eqnarray}
\bar{f}_{q,1}   =
\begin{cases}
\min \lp 0,  \frac{\lp\g_i^{(1)}\rp^2}{4\bar{\gamma}_1} -\frac{(\max(\g_i^{(1)}-2\h_i^{(2)}\bar{\gamma}_1,0)).^2}{4\bar{\gamma}_1(1+4\gamma\bar{\gamma}_1)}  \rp , & \mbox{if } \g_i^{(1)}\leq 0 \\
   \frac{\lp\g_i^{(1)}\rp^2}{4\bar{\gamma}_1} -\frac{(\max(\g_i^{(1)}-2\h_i^{(2)}\bar{\gamma}_1,0)).^2}{4\bar{\gamma}_1(1+4\gamma\bar{\gamma}_1)} , & \mbox{otherwise},
\end{cases}
\label{eq:plhrd12}
\end{eqnarray}
and\
\begin{eqnarray}
\nonumber\\
\bar{f}_{q,2}
  =
  \begin{cases}
    - \lp \g_i^{(2)}\rp^2 -\nu_1 , & \mbox{if } \g_i^{(2)}\leq 0 \\
    -\nu_1,  & \mbox{if } 0\leq \g_i^{(2)} \leq \sqrt{2\nu_1} \\
    -\lp \g_i^{(2)}\rp^2 +\nu_1, &  \mbox{otherwise}.
  \end{cases}
\label{eq:plhrd12a0}
 \end{eqnarray}
After setting
\begin{eqnarray}\label{eq:plnegprac19a1a1}
  \bar{a}_1     &=    &  \sqrt{2\nu_1}\sqrt{1+\frac{2c_3r^2}{4\bar{\gamma}_2}} \nonumber \\
\bar{f}_{21}^+ & = &   \frac{ e^{c_3 \frac{r^2}{4\bar{\gamma}_2} \nu_1} }  {2} \nonumber \\
\bar{f}_{22}^+ & = &   \frac{ e^{-c_3 \frac{r^2}{4\bar{\gamma}_2} \nu_1} }  {2} \erfc\lp \frac{ \sqrt{2\nu_1}}{\sqrt{2}}  \rp\nonumber \\
\bar{f}_{23}^+ & = &  \frac{e^{c_3 \frac{r^2}{4\bar{\gamma}_2} \nu_1}} {\sqrt{1+\frac{2c_3 r^2}{4\bar{\gamma}_2}}}  \lp \frac{1}{2}- \frac{1}{2} \erfc\lp \frac{ \bar{a}_1}{\sqrt{2}}  \rp  \rp,
 \end{eqnarray}
and solving the integrals one obtains
  \begin{eqnarray}\label{eq:plhrd10a0}
f_{q,2}^{(lift)}(\bar{\gamma}_2,r)=\mE_{\cG_{(2)}} e^{ - c_3 \bar{f}_{q,2}  }=\bar{f}_{21}^+ +\bar{f}_{22}^+ + \bar{f}_{23}^+.
  \end{eqnarray}
  On the other hand, after setting for $\g_i^{(1)} > 0$
  \begin{eqnarray}
        \bar{A}^+ & = & \frac{ \g_I^{(1)}    } { 2\bar{\gamma}_1} \nonumber \\
         \bar{B}^+&  = &   \frac{ c_3\bar{\gamma}_1}  {   (1+4\gamma \bar{\gamma}_1) }  \nonumber \\
         \bar{C}^+ & = & \bar{A}^+ \nonumber \\
        I_{11}^+ & = &
         \frac{1}{2\sqrt{1-2\bar{B}^+}} e^{ \frac{\bar{B}^+}{1-2\bar{B}^+} \lp \bar{A}^+\rp^2 }
         \erfc\lp  - \frac{ (2\bar{B}^+ \bar{A}^+  +  (1-2\bar{B}^+) \bar{C}^+) }  {  \sqrt{  2(1-2\bar{B}^+) }  }    \rp
         \nonumber \\
        \hat{f}_{q,1}^+ & = &   e^{-c_3  \frac{ (\g_i^{(1)}   )^2}  {4\bar{\gamma}_1}} \lp  \lp 1-\frac{1}{2}\erfc\lp  -\frac{\bar{C}^+} {\sqrt{2} } \rp \rp  + I_{11}^+ \rp.
  \label{eq:plhrd13}
 \end{eqnarray}
and for $\g_i^{(1)}\leq 0$
\begin{eqnarray}
        \bar{A}^+ & = & \frac{ \g_I^{(1)}    } { 2\bar{\gamma}_1} \nonumber \\
         \bar{B}^+&  = &   \frac{c_3 \bar{\gamma}_1}  {   (1+4\gamma \bar{\gamma}_1) }  \nonumber \\
         \bar{C}^+& =  & \lp   \g_I^{(1)}    -  \sqrt{  |  -(\g_I^{(1)}    )^2 |  (1+4\gamma\bar{\gamma}_1) } \rp   \frac{1}{2\bar{\gamma}_1}
  \nonumber \\
        I_{11}^+ & = &
         \frac{1}{2\sqrt{1-2\bar{B}^+}} e^{ \frac{\bar{B}^+}{1-2\bar{B}^+} \lp \bar{A}^+\rp^2 }
         \erfc\lp  - \frac{ (2\bar{B}^+ \bar{A}^+  +  (1-2\bar{B}^+) \bar{C}^+) }  {  \sqrt{  2(1-2\bar{B}^+) }  }    \rp
         \nonumber \\
         I_{12}^+& = &   e^{-c_3  \frac{ (\g_i^{(1)}   )^2}  {4\bar{\gamma}_1}}  I_{11}^+
         \nonumber \\
         I_{13}^+ & = &    1-\frac{1}{2}\erfc\lp  -\frac{\bar{C}^+} {\sqrt{2} } \rp
         \nonumber \\
          \hat{f}_{q,1}^+   & = & I_{12}^+ +  I_{13}^+,
  \label{eq:plhrd14}
 \end{eqnarray}
 solving remaining integrals gives
\begin{eqnarray}
f_{q,1}^{(lift)}(c_3,\bar{\gamma}_1,\gamma)  & = &  \mE_{\cG_{(2)}} e^{-c_3 \bar{f}_{q,1}}
  =
\int_{\g_i^{(1)}>0 }
\hat{f}_{q,1}^+ \frac{e^{-\frac{ \lp \g_i^{(1)} \rp^2     } {2}  }}{\sqrt{2\pi}} d\g_i^{(1)}
  +
\int_{\g_i^{(1)}\leq 0 }
\hat{f}_{q,2}^+ \frac{e^{-\frac{ \lp \g_i^{(1)} \rp^2         } {2}  }}{\sqrt{2\pi}} d\g_i^{(1)}.
  \label{eq:plhrd15}
 \end{eqnarray}
After solving the integral over $\h_i^{(1)}$ and optimizing over $\gamma_{sph}$ one first finds (see, e.g.,  \cite{StojnicMoreSophHopBnds10})
\begin{eqnarray}
\hat{\gamma}_{sph} =\frac{c_3+\sqrt{c_3+4}}{4},
  \label{eq:plhrd15a0}
 \end{eqnarray}
 and then  rewrites (\ref{eq:plhrd6}) as
\begin{eqnarray}
 \bar{\phi}_0
& \geq &
   \min_{r>0,\bar{\gamma}_1,\bar{\gamma}_2} \max_{\nu,\gamma}
\Bigg .\Bigg(
 \frac{c_3}{2} + \frac{c_3}{2} r^2
 +\bar{\gamma}_1
 -\frac{\alpha_1}{c_3} \log \lp f_{q,1}^{(lift)}\rp
-\gamma r^2
 +\bar{\gamma}_2 - \nu(2-\alpha_2)
 \nonumber \\
 & &
-\frac{\alpha_2}{c_3} \log \lp f_{q,2}^{(lift)} \rp
- \hat{\gamma}_{sph}  +\frac{1}{2c_3} \log \lp  1  -  \frac{c_3}{2\hat{\gamma}_{sph}}     \rp
\Bigg.\Bigg),
\label{eq:plhrd16}
 \end{eqnarray}
with $f_{q,1}^{(lift)}$, $f_{q,2}^{(lift)}$, and $\hat{\gamma}_{sph}$  as in   (\ref{eq:plhrd10a0}), (\ref{eq:plhrd15})), and  (\ref{eq:plhrd15a0}), respectively. As (\ref{eq:plhrd16}) holds for any $c_3$, maximization of the right hand side over $c_3$ also gives
\begin{eqnarray}
 \bar{\phi}_0
& \geq &
 \max_{c_3> 0}  \min_{r>0,\bar{\gamma}_1,\bar{\gamma}_2} \max_{\nu,\gamma}
\Bigg .\Bigg(
 \frac{c_3}{2} + \frac{c_3}{2} r^2
 +\bar{\gamma}_1
 -\frac{\alpha_1}{c_3} \log \lp f_{q,1}^{(lift)}\rp
-\gamma r^2
 +\bar{\gamma}_2 - \nu(2-\alpha_2)
 \nonumber \\
 & &
-\frac{\alpha_2}{c_3} \log \lp f_{q,2}^{(lift)} \rp
- \hat{\gamma}_{sph}  +\frac{1}{2c_3} \log \lp  1  -  \frac{c_2}{2\hat{\gamma}_{sph}}     \rp
\Bigg.\Bigg).
\label{eq:plhrd16a0}
 \end{eqnarray}
Limiting value of
$\alpha_2$ for which one has $\bar{\phi}_0=0$ gives then an upper bound on the injectivity capacity. The concrete parameters values obtained through numerical evaluations are given in Table \ref{tab:tab7}. The plain RDT results are shown in parallel as well.
\begin{table}[h]
\caption{ Lifted RDT parameters; 2-layer ReLU NN \emph{weak} injectivity capacity ; $n\rightarrow\infty$; }\vspace{.1in}
\centering
\def\arraystretch{1.2}
\begin{tabular}{||l||c||c||c|c||c||c||c||c||}\hline\hline
  RDT parameters                              &  $\alpha_1$ &   $r$     & $\bar{\gamma}_1$  & $\bar{\gamma}_2$ &  $\nu$ & $\gamma$
   & $c_3$
   & $\alpha_{ReLU}^{(inj)}(\alpha_1)$  \\ \hline\hline
Plain RDT                                      & $6.7004$  & $1.7697$   & $0.8935$  & $0.9642$ & $0.5560$  & $0.3078$
 & $\rightarrow 0$
 & \bl{$\mathbf{8.267}$}
  \\ \hline\hline
Lifted RDT                                      & $6.7004$  & $1.7931$   & $0.8810$  & $0.9053$ & $0.5504$  & $0.3361$
 & $ 0.1091$
 & \bl{$\mathbf{8.264}$}
  \\ \hline\hline
  \end{tabular}
\label{tab:tab7}
\end{table}
The upper bound is indeed lowered. To get the corresponding strong  injectivity capacity bound, instead of (\ref{eq:plhrd16}) we utilize
\begin{eqnarray}
 \bar{\phi}_0
& \geq &
 \max_{c_3> 0}  \min_{r>0,\bar{\gamma}_1,\bar{\gamma}_2} \max_{\nu,\gamma}
\Bigg .\Bigg(
 \frac{c_3}{2} + \frac{c_3}{2} r^2
 +\bar{\gamma}_1
 -\frac{\alpha_1}{c_3} \log \lp f_{q,1}^{(lift)}\rp
-\gamma r^2
 +\bar{\gamma}_2 - \nu(4-\alpha_2)
 \nonumber \\
 & &
-\frac{\alpha_2}{c_3} \log \lp f_{q,2}^{(lift)} \rp
- \hat{\gamma}_{sph}  +\frac{1}{2c_3} \log \lp  1  -  \frac{c_2}{2\hat{\gamma}_{sph}}     \rp
\Bigg.\Bigg).
\label{eq:plhrd17}
 \end{eqnarray}
The obtained results together with the corresponding plain RDT ones are shown in Table \ref{tab:tab8}. We again observe a lowering of the plain RDT upper bound. In fact, the lowering is now a bit more pronounced than in the weak case.
\begin{table}[h]
\caption{ Lifted RDT parameters; 2-layer ReLU NN \emph{strong} injectivity capacity ; $n\rightarrow\infty$; }\vspace{.1in}
\centering
\def\arraystretch{1.2}
\begin{tabular}{||l||c||c||c|c||c||c||c||c||}\hline\hline
  RDT parameters                               &  $\alpha_1$ &   $r$     & $\bar{\gamma}_1$  & $\bar{\gamma}_2$ &  $\nu$ & $\gamma$    & $c_3$ & $\alpha_{ReLU}^{(inj)}(\alpha_1)$  \\ \hline\hline
Plain RDT                                      & $6.7004$  & $1.7708$   & $ 0.9647$  & $0.8938$ & $0.3954$  & $0.3077$   & $ \rightarrow 0$  & \bl{$\mathbf{12.350}$}
  \\ \hline\hline
Lifted RDT                                      & $6.7004$  & $ 1.9060$   & $0.7862$  & $0.5707$ & $0.3561$  & $ 0.5728$   & $0.8315$   & \bl{$\mathbf{12.183}$}
  \\ \hline\hline
  \end{tabular}
\label{tab:tab8}
\end{table}

\subsection{Lifted RDT -- $l$-layer ReLU NN}
 \label{sec:plllay}

The above 2-layer considerations extend to general  $l$-layer depth precisely in the same manner as  the corresponding plain RDT ones. After setting $r^{(0)}=1$, $\alpha_l^{(0)}=\alpha_l$, $\bar{\gamma}_l^{(0)}=\bar{\gamma}_l$, we recall on (\ref{eq:3layhrd18})
  \begin{eqnarray}
\alpha^{(0)}&  = &
\begin{bmatrix}  \alpha_1 \\
                   \alpha_2 \\
                   \vdots \\
                   \alpha_{l-1} \end{bmatrix}, \quad
r^{(0)}  = \begin{bmatrix}
                                                                      r \\
                                                                      r_2 \\
                                                                      \vdots \\
                                                                      r_{l-1}
                                                                    \end{bmatrix},
                                                                    \quad
                                                                    \bar{\gamma}^{(0)}  = \begin{bmatrix}
                                                                      \bar{\gamma}_1 \\
                                                                      \bar{\gamma}_2 \\
                                                                      \vdots \\
                                                                      \bar{\gamma}_{l-1}
                                                                    \end{bmatrix},
                                                                    \quad
\gamma^{(0)}=\begin{bmatrix}
                                                                      \gamma \\
                                                                      \gamma_2 \\
                                                                      \vdots \\
                                                                      \gamma_{l-1}
                                                                    \end{bmatrix},
\label{eq:pl3layhrd18}
\end{eqnarray}
and analogously to (\ref{eq:3layhrd18a0})  write
 \begin{eqnarray}
 f_{q,1}^{(lift,i)} & =  & f_{q,1} (c_3r^{(0)}_{i-1},\bar{\gamma}_i^{(0)}, \gamma_i^{(0}) ),1\leq i\leq l-1 \nonumber \\
 f_{q,2}^{(lift,l)} & =  & f_{q,2} (c_3r^{(0)}_{l-1},\bar{\gamma}_l^{(0)}, r_{l-1}^{(0}) ).
\label{eq:pl3layhrd18a0}
 \end{eqnarray}
It is then straightforward to mimic the mechanism  utilized to obtain (\ref{eq:3layhrd19}) and combine it with (\ref{eq:plhrd16a0}) to write  l-layer analogue to (\ref{eq:plhrd16a0})
\begin{eqnarray}
 \bar{\phi}_0^{(l)}
& \geq &
\hspace{-.01in}   \max_{c_3>0} \min_{r^{(0)},\bar{\gamma}_1^{(0)}} \max_{\nu,\gamma^{(0)}}
\Bigg.\Bigg(
 \sum_{i=1}^{l-1} \lp \bar{\gamma}_{i}^{(0)}r_{i-1}^{(0)} +  \frac{\alpha_{i}^{(0)}}{c_3}\log \lp f_{q,1}^{(lift,i)}  \rp
    -  \gamma_i^{(0)} \frac{\lp r_{i}^{(0)}\rp ^2}{r_{i-1}^{(0)}} \rp
 +   \bar{\gamma}_l^{(0)} + \frac{\alpha_l^{(0)}}{c_3} \log \lp f_{q,2}^{(lift,l)}  \rp
\nonumber \\
& &
 -\nu(2-\alpha_l^{(0)}) - \hat{\gamma}_{sph} +  \frac{1}{2c_3} \log \lp 1 - \frac{c_3}{2\hat{\gamma}_{sph}} \rp \Bigg.\Bigg).
\label{eq:pl3layhrd19}
 \end{eqnarray}
The strong injectivity complement is then easily obtained as
\begin{eqnarray}
 \bar{\phi}_0^{(l)}
& \geq &
\hspace{-.01in}   \max_{c_3>0} \min_{r^{(0)},\bar{\gamma}_1^{(0)}} \max_{\nu,\gamma^{(0)}}
\Bigg.\Bigg(
 \sum_{i=1}^{l-1} \lp \bar{\gamma}_{i}^{(0)}r_{i-1}^{(0)} +  \frac{\alpha_{i}^{(0)}}{c_3}\log \lp f_{q,1}^{(lift,i)}  \rp
    -  \gamma_i^{(0)} \frac{\lp r_{i}^{(0)}\rp ^2}{r_{i-1}^{(0)}} \rp
 +   \bar{\gamma}_l^{(0)} + \frac{\alpha_l^{(0)}}{c_3} \log \lp f_{q,2}^{(lift,l)}  \rp
\nonumber \\
& &
 -\nu(4-\alpha_l^{(0)}) - \hat{\gamma}_{sph} +  \frac{1}{2c_3} \log \lp 1 - \frac{c_3}{2\hat{\gamma}_{sph}} \rp \Bigg.\Bigg).
\label{eq:pl3layhrd20}
 \end{eqnarray}
After  conducting  numerical evaluations  we did not find substantial improvements beyond the second layer.

\section{Conclusion}
\label{sec:conc}

We studied the injectivity capacity of a deep ReLU network. For 1 -layer networks studying injectivity capacity was recently shown \cite{Stojnicinjrelu24}   to be equivalent to studying the capacity of the  so-called  $\ell_0$ spherical perceptron.  Here we show that similar concept applies for studying multilayered deep ReLU networks. Namely, the injectivity of $l$-layer deep ReLU NN is uncovered to directly relate to $l$-extened  $\ell_0$ spherical perceptrons. Utilizing  \emph{Random duality theory} (RDT) we then create a generic program for statistical performance analysis of such perceptrons  and consequently the injectivity properties of deep ReLU nets. As a result we obtain upper bounds on the injectivity capacity for any number of network layers and any number of nodes in any of the layers.

Two notions of injectivity capacity, \emph{weak} and \emph{strong}, are introduced and concrete numerical values for both of them are obtained. We observe a tendency that relative network expansion per layer (increase from the number of layer's inputs to the number of layer's outputs)  needed for injectivity decreases in any new added layer. Moreover, this decreasing is more pronounced for, practically more relevant, weak injectivity. In particular, we obtain that 4-layer deep nets are already close to reaching maximally needed expansion which is $\sim 10$.

As the results obtained through plain RDT are upper-bounds, we implemented a partially lifted (pl) RDT variant to lower them. To do so we first created a generic analytical adaptation of the plain RDT program and then conducted related numerical evaluations. For 2-layer nets we found that partially lifted RDT indeed is capable of decreasing the plain RDT upper bounds.

 Since the developed methodologies are fairly generic,  many extensions and generalizations are also possible. Beyond naturally expected fl RDT implementations, these, for example, include studying stability properties, injectivity sensitivity with respect to various imperfections (noisy gates, missing/incomplete portions of the training data, not fully adequately or optimally trained networks, etc.), and many others. Given the importance of injectivity within the \emph{deep compressed sensing} paradigm, quantifying networks generalizability properties as well as the complexity of the training and recovery algorithms related to such contexts are of great interest as well. The mechanisms presented here can be used for any of these tasks. As the associated technicalities are problem specific, we discuss them in separate papers.

\begin{singlespace}
\bibliographystyle{plain}
\bibliography{nflgscompyxRefs}

\begin{thebibliography}{10}

\bibitem{AbbLiSly21b}
E.~Abbe, S.~Li, and A.~Sly.
\newblock Proof of the contiguity conjecture and lognormal limit for the
  symmetric perceptron.
\newblock In {\em 62nd {IEEE} Annual Symposium on Foundations of Computer
  Science, {FOCS} 2021, Denver, CO, USA, February 7-10, 2022}, pages 327--338.
  {IEEE}, 2021.

\bibitem{AbbLiSly21a}
E.~Abbe, S.~Li, and A.~Sly.
\newblock Binary perceptron: efficient algorithms can find solutions in a rare
  well-connected cluster.
\newblock In {\em {STOC} '22: 54th Annual {ACM} {SIGACT} Symposium on Theory of
  Computing, Rome, Italy, June 20 - 24, 2022}, pages 860--873. {ACM}, 2022.

\bibitem{AlwLiuSaw21}
R.~Alweiss, Y.~P. Liu, and M.~Sawhney.
\newblock Discrepancy minimization via a self-balancing walk.
\newblock {\em In Proc. 53rd STOC, ACM}, pages 14--20, 2021.

\bibitem{ALMT14}
D.~Amelunxen, M.~Lotz, M.~McCoy, and J.~Tropp.
\newblock Living on the edge: phase transitions in convex programs with random
  data.
\newblock {\em Information and Inference: A Journal of the IMA}, 3(3):224--294,
  2014.

\bibitem{AMZ24}
B.~L. Annesi, E.~M. Malatesta, and F.~Zamponi.
\newblock Exact full-{RSB} {SAT}/{UNSAT} transition in infinitely wide
  two-layer neural networks.
\newblock 2023.
\newblock available online at
  {\small\bl{\url{http://arxiv.org/abs/2410.06717}}}.

\bibitem{ArridgeMOS19}
S.~R. Arridge, P.~Maass, O.~{\"{O}}ktem, and C.~B. Sch{\"{o}}nlieb.
\newblock Solving inverse problems using data-driven models.
\newblock {\em Acta Numer.}, 28:1--174, 2019.

\bibitem{AubPerZde19}
B.~Aubin, W.~Perkins, and L.~Zdeborova.
\newblock Storage capacity in symmetric binary perceptrons.
\newblock {\em J. Phys. A}, 52(29):294003, 2019.

\bibitem{BalMalZech19}
C.~Baldassi, E.~M. Malatesta, and R.~Zecchina.
\newblock Properties of the geometry of solutions and capacity of multilayer
  neural networks with rectified linear unit activations.
\newblock {\em Phys. Rev. Lett.}, 123:170602, October 2019.

\bibitem{BalVen87}
P.~Baldi and S.~Venkatesh.
\newblock Number od stable points for spin-glasses and neural networks of
  higher orders.
\newblock {\em Phys. Rev. Letters}, 58(9):913--916, Mar. 1987.

\bibitem{BHS92}
E.~Barkai, D.~Hansel, and H.~Sompolinsky.
\newblock Broken symmetries in multilayered perceptrons.
\newblock {\em Phys. Rev. A}, 45(6):4146, March 1992.

\bibitem{BoraJPD17}
A.~Bora, A.~Jalal, E.~Price, and A.~G. Dimakis.
\newblock Compressed sensing using generative models.
\newblock In {\em Proceedings of the 34th International Conference on Machine
  Learning, {ICML} 2017, Sydney, NSW, Australia, 6-11 August 2017}, volume~70
  of {\em Proceedings of Machine Learning Research}, pages 537--546. {PMLR},
  2017.

\bibitem{Cameron60}
S.~H. Cameron.
\newblock Tech-report 60-600.
\newblock {\em Proceedings of the bionics symposium}, pages 197--212, 1960.
\newblock Wright air development division, {D}ayton, {O}hio.

\bibitem{Clum22}
C.~Clum.
\newblock {\em Topics in the Mathematics of Data Science}.
\newblock PhD thesis, The Ohio State University, 2022.

\bibitem{Cover65}
T.~Cover.
\newblock Geomretrical and statistical properties of systems of linear
  inequalities with applications in pattern recognition.
\newblock {\em IEEE Transactions on Electronic Computers}, (EC-14):326--334,
  1965.

\bibitem{DaskalakisRZ20a}
C.~Daskalakis, D.~Rohatgi, and E.~Zampetakis.
\newblock Constant-expansion suffices for compressed sensing with generative
  priors.
\newblock In {\em Advances in Neural Information Processing Systems 33: Annual
  Conference on Neural Information Processing Systems 2020, NeurIPS 2020,
  December 6-12, 2020, virtual}, 2020.

\bibitem{DharGE18}
M.~Dhar, A.~Grover, and S.~Ermon.
\newblock Modeling sparse deviations for compressed sensing using generative
  models.
\newblock In {\em Proceedings of the 35th International Conference on Machine
  Learning, {ICML} 2018, Stockholmsm{\"{a}}ssan, Stockholm, Sweden, July 10-15,
  2018}, volume~80 of {\em Proceedings of Machine Learning Research}, pages
  1222--1231. {PMLR}, 2018.

\bibitem{DonohoPol}
D.~Donoho.
\newblock High-dimensional centrally symmetric polytopes with neighborlines
  proportional to dimension.
\newblock {\em Disc. Comput. Geometry}, 35(4):617--652, 2006.

\bibitem{EKTVZ92}
A.~Engel, H.~M. Kohler, F.~Tschepke, H.~Vollmayr, and A.~Zippelius.
\newblock Storage capacity and learning algorithms for two-layer neural
  networks.
\newblock {\em Phys. Rev. A}, 45(10):7590, May 1992.

\bibitem{EstrachSL14}
J.~B. Estrach, A.~Szlam, and Y.~LeCun.
\newblock Signal recovery from pooling representations.
\newblock In {\em Proceedings of the 31th International Conference on Machine
  Learning, {ICML} 2014, Beijing, China, 21-26 June 2014}, volume~32 of {\em
  {JMLR} Workshop and Conference Proceedings}, pages 307--315. JMLR.org, 2014.

\bibitem{FazlyabRHMP19}
M.~Fazlyab, A.~Robey, H.~Hassani, M.~Morari, and G.~J. Pappas.
\newblock Efficient and accurate estimation of lipschitz constants for deep
  neural networks.
\newblock In {\em Advances in Neural Information Processing Systems 32: Annual
  Conference on Neural Information Processing Systems 2019, NeurIPS 2019,
  December 8-14, 2019, Vancouver, BC, Canada}, pages 11423--11434, 2019.

\bibitem{FletcherRS18}
Alyson~K. Fletcher, Sundeep Rangan, and Philip Schniter.
\newblock Inference in deep networks in high dimensions.
\newblock In {\em 2018 {IEEE} International Symposium on Information Theory,
  {ISIT} 2018, Vail, CO, USA, June 17-22, 2018}, pages 1884--1888. {IEEE},
  2018.

\bibitem{FPSUZ17}
S.~Franz, G.~Parisi, M.~Sevelev, P.~Urbani, and F.~Zamponi.
\newblock Universality of the {SAT-UNSAT} (jamming) threshold in non-convex
  continuous constraint satisfaction problems.
\newblock {\em SciPost Physics}, 2:019, 2017.

\bibitem{FuruyaPLH23}
T.~Furuya, M.~Puthawala, M.~Lassas, and M.~V. de~Hoop.
\newblock Globally injective and bijective neural operators.
\newblock In {\em Advances in Neural Information Processing Systems 36: Annual
  Conference on Neural Information Processing Systems 2023, NeurIPS 2023, New
  Orleans, LA, USA, December 10 - 16, 2023}, 2023.

\bibitem{MitchDurb89}
R.~M.~Durbin G.~J.~Mitchison.
\newblock Bounds on the learning capacity of some multi-layer networks.
\newblock {\em Biological Cybernetics}, 60:345--365, 1989.

\bibitem{GamKizPerXu22}
David Gamarnik, Eren~C. Kizildag, Will Perkins, and Changji Xu.
\newblock Algorithms and barriers in the symmetric binary perceptron model.
\newblock In {\em 63rd {IEEE} Annual Symposium on Foundations of Computer
  Science, {FOCS} 2022, Denver, CO, USA, October 31 - November 3, 2022}, pages
  576--587. {IEEE}, 2022.

\bibitem{Gar88}
E.~Gardner.
\newblock The space of interactions in neural networks models.
\newblock {\em J. Phys. A: Math. Gen.}, 21:257--270, 1988.

\bibitem{GarDer88}
E.~Gardner and B.~Derrida.
\newblock Optimal storage properties of neural networks models.
\newblock {\em J. Phys. A: Math. Gen.}, 21:271--284, 1988.

\bibitem{Gordon85}
Y.~Gordon.
\newblock Some inequalities for {G}aussian processes and applications.
\newblock {\em Israel Journal of Mathematics}, 50(4):265--289, 1985.

\bibitem{Gordon88}
Y.~Gordon.
\newblock On {M}ilman's inequality and random subspaces which escape through a
  mesh in ${R}^n$.
\newblock {\em Geometric Aspect of of functional analysis, Isr. Semin. 1986-87,
  Lect. Notes Math}, 1317, 1988.

\bibitem{GoukFPC21}
H.~Gouk, E.~Frank, B.~Pfahringer, and M.~J. Cree.
\newblock Regularisation of neural networks by enforcing lipschitz continuity.
\newblock {\em Mach. Learn.}, 110(2):393--416, 2021.

\bibitem{GutSte90}
H.~Gutfreund and Y.~Stein.
\newblock Capacity of neural networks with discrete synaptic couplings.
\newblock {\em J. Physics A: Math. Gen}, 23:2613, 1990.

\bibitem{HandLV18}
P.~Hand, O.~Leong, and V.~Voroninski.
\newblock Phase retrieval under a generative prior.
\newblock In {\em Advances in Neural Information Processing Systems 31: Annual
  Conference on Neural Information Processing Systems 2018, NeurIPS 2018,
  December 3-8, 2018, Montr{\'{e}}al, Canada}, pages 9154--9164, 2018.

\bibitem{HeWJ17}
H.~He, C.{-}K. Wen, and S.~Jin.
\newblock Generalized expectation consistent signal recovery for nonlinear
  measurements.
\newblock In {\em 2017 {IEEE} International Symposium on Information Theory,
  {ISIT} 2017, Aachen, Germany, June 25-30, 2017}, pages 2333--2337. {IEEE},
  2017.

\bibitem{HHHV18}
R.~Heckel, W.~Huang, P.~Hand, and V.~Voroninski.
\newblock Deep denoising: Rate-optimal recovery of structured signals with a
  deep prior.
\newblock 2018.
\newblock available online at \bl{\url{http://arxiv.org/abs/1805.08855}}.

\bibitem{HegdeWB07}
C.~Hegde, M.~B. Wakin, and R.~G. Baraniuk.
\newblock Random projections for manifold learning.
\newblock In {\em Advances in Neural Information Processing Systems 20,
  Proceedings of the Twenty-First Annual Conference on Neural Information
  Processing Systems, Vancouver, British Columbia, Canada, December 3-6, 2007},
  pages 641--648. Curran Associates, Inc., 2007.

\bibitem{HHHV18a}
W.~Huang, P.~Hand, R.~Heckel, and V.~Voroninski.
\newblock A provably convergent scheme for compressive sensing under random
  generative priors.
\newblock 2018.
\newblock available online at \bl{\url{http://arxiv.org/abs/1812.04176}}.

\bibitem{JordanD20}
M.~Jordan and A.~G. Dimakis.
\newblock Exactly computing the local lipschitz constant of relu networks.
\newblock In {\em Advances in Neural Information Processing Systems 33: Annual
  Conference on Neural Information Processing Systems 2020, NeurIPS 2020,
  December 6-12, 2020, virtual}, 2020.

\bibitem{Joseph60}
R.~D. Joseph.
\newblock The number of orthants in $n$-space instersected by an
  $s$-dimensional subspace.
\newblock {\em Tech. memo 8, project {PARA}}, 1960.
\newblock Cornel aeronautical lab., Buffalo, N.Y.

\bibitem{KothariKHD21}
K.~Kothari, A.~Khorashadizadeh, M.~V. de~Hoop, and I.~Dokmanic.
\newblock Trumpets: Injective flows for inference and inverse problems.
\newblock In {\em Proceedings of the Thirty-Seventh Conference on Uncertainty
  in Artificial Intelligence, {UAI} 2021, Virtual Event, 27-30 July 2021},
  volume 161 of {\em Proceedings of Machine Learning Research}, pages
  1269--1278. {AUAI} Press, 2021.

\bibitem{KraMez89}
W.~Krauth and M.~Mezard.
\newblock Storage capacity of memory networks with binary couplings.
\newblock {\em J. Phys. France}, 50:3057--3066, 1989.

\bibitem{LeiJDD19}
Q.~Lei, A.~Jalal, I.~S. Dhillon, and A.~G. Dimakis.
\newblock Inverting deep generative models, one layer at a time.
\newblock In {\em Advances in Neural Information Processing Systems 32: Annual
  Conference on Neural Information Processing Systems 2019, NeurIPS 2019,
  December 8-14, 2019, Vancouver, BC, Canada}, pages 13910--13919, 2019.

\bibitem{LouartLC17}
C.~Louart, Z.~Liao, and R.~Couillet.
\newblock A random matrix approach to neural networks.
\newblock {\em CoRR}, abs/1702.05419, 2017.

\bibitem{MBBDN23}
A.~Maillard, A.~S. Bandeira, D.~Belius, I.~Dokmanic, and S.~Nakajima.
\newblock Injectivity of relu networks: perspectives from statistical physics.
\newblock 2023.
\newblock available online at
  {\small\bl{\url{http://arxiv.org/abs/2302.14112}}}.

\bibitem{MardaniSDPMVP18}
M.~Mardani, Q.~Sun, D.~L. Donoho, V.~Papyan, H.~Monajemi, S.~Vasanawala, and
  J.~M. Pauly.
\newblock Neural proximal gradient descent for compressive imaging.
\newblock In {\em Advances in Neural Information Processing Systems 31: Annual
  Conference on Neural Information Processing Systems 2018, NeurIPS 2018,
  December 3-8, 2018, Montr{\'{e}}al, Canada}, pages 9596--9606, 2018.

\bibitem{Pal21}
D.~Paleka.
\newblock {\em Injectivity of ReLU neural networks at initialization}.
\newblock Master thesis, ETH Zurich, 2021.

\bibitem{PenningtonW17}
J.~Pennington and P.~Worah.
\newblock Nonlinear random matrix theory for deep learning.
\newblock In {\em Advances in Neural Information Processing Systems 30: Annual
  Conference on Neural Information Processing Systems 2017, December 4-9, 2017,
  Long Beach, CA, {USA}}, pages 2637--2646, 2017.

\bibitem{PerkXu21}
W.~Perkins and C.~Xu.
\newblock Frozen 1-{RSB} structure of the symmetric {I}sing perceptron.
\newblock {\em {STOC} 2021: Proceedings of the 53rd Annual {ACM SIGACT}
  {S}ymposium on {T}heory of {C}omputing}, pages 1579--1588, 2021.

\bibitem{PuthawalaKLDH22}
M.~Puthawala, K.~Kothari, M.~Lassas, I.~Dokmanic, and M.~V. de~Hoop.
\newblock Globally injective relu networks.
\newblock {\em J. Mach. Learn. Res.}, 23:105:1--105:55, 2022.

\bibitem{PuthawalaLDH22}
M.~Puthawala, M.~Lassas, I.~Dokmanic, and M.~V. de~Hoop.
\newblock Universal joint approximation of manifolds and densities by simple
  injective flows.
\newblock In {\em International Conference on Machine Learning, {ICML} 2022,
  17-23 July 2022, Baltimore, Maryland, {USA}}, volume 162 of {\em Proceedings
  of Machine Learning Research}, pages 17959--17983. {PMLR}, 2022.

\bibitem{RanganSF17}
S.~Rangan, P.~Schniter, and A.~K. Fletcher.
\newblock Vector approximate message passing.
\newblock In {\em 2017 {IEEE} International Symposium on Information Theory,
  {ISIT} 2017, Aachen, Germany, June 25-30, 2017}, pages 1588--1592. {IEEE},
  2017.

\bibitem{RomanoEM17}
Y.~Romano, M.~Elad, and P.~Milanfar.
\newblock The little engine that could: Regularization by denoising {(RED)}.
\newblock {\em {SIAM} J. Imaging Sci.}, 10(4):1804--1844, 2017.

\bibitem{RossC21}
B.~Leigh Ross and J.~C. Cresswell.
\newblock Tractable density estimation on learned manifolds with conformal
  embedding flows.
\newblock In {\em Advances in Neural Information Processing Systems 34: Annual
  Conference on Neural Information Processing Systems 2021, NeurIPS 2021,
  December 6-14, 2021, virtual}, pages 26635--26648, 2021.

\bibitem{Schlafli}
L.~Schlafli.
\newblock {\em Gesammelte Mathematische AbhandLungen I}.
\newblock Basel, Switzerland: Verlag Birkhauser, 1950.

\bibitem{SW08}
R.~Schneider and W.~Weil.
\newblock {\em Stochastic and Integral Geometry}.
\newblock Springer-Verlag Berlin Heidelberg, 2008.

\bibitem{SchniterRF16}
P.~Schniter, S.~Rangan, and A.~K. Fletcher.
\newblock Vector approximate message passing for the generalized linear model.
\newblock In {\em 50th Asilomar Conference on Signals, Systems and Computers,
  {ACSSC} 2016, Pacific Grove, CA, USA, November 6-9, 2016}, pages 1525--1529.
  {IEEE}, 2016.

\bibitem{ShahH18}
V.~Shah and C.~Hegde.
\newblock Solving linear inverse problems using gan priors: An algorithm with
  provable guarantees.
\newblock In {\em 2018 {IEEE} International Conference on Acoustics, Speech and
  Signal Processing, {ICASSP} 2018, Calgary, AB, Canada, April 15-20, 2018},
  pages 4609--4613. {IEEE}, 2018.

\bibitem{SchTir02}
M.~Shcherbina and B.~Tirozzi.
\newblock On the volume of the intrersection of a sphere with random half
  spaces.
\newblock {\em C. R. Acad. Sci. Paris. Ser I}, (334):803--806, 2002.

\bibitem{SchTir03}
M.~Shcherbina and B.~Tirozzi.
\newblock Rigorous solution of the {G}ardner problem.
\newblock {\em Comm. on Math. Physics}, (234):383--422, 2003.

\bibitem{StojnicCSetamBlock09}
M.~Stojnic.
\newblock Block-length dependent thresholds in block-sparse compressed sensing.
\newblock available online at \bl{\url{http://arxiv.org/abs/0907.3679}}.

\bibitem{StojnicCSetam09}
M.~Stojnic.
\newblock Various thresholds for $\ell_1$-optimization in compressed sensing.
\newblock available online at \bl{\url{http://arxiv.org/abs/0907.3666}}.

\bibitem{StojnicICASSP10block}
M.~Stojnic.
\newblock Block-length dependent thresholds for $\ell_2/\ell_1$-optimization in
  block-sparse compressed sensing.
\newblock {\em ICASSP, IEEE International Conference on Acoustics, Signal and
  Speech Processing}, pages 3918--3921, 14-19 March 2010.
\newblock Dallas, TX.

\bibitem{StojnicICASSP10var}
M.~Stojnic.
\newblock $\ell_1$ optimization and its various thresholds in compressed
  sensing.
\newblock {\em ICASSP, IEEE International Conference on Acoustics, Signal and
  Speech Processing}, pages 3910--3913, 14-19 March 2010.
\newblock Dallas, TX.

\bibitem{StojnicGardGen13}
M.~Stojnic.
\newblock Another look at the {G}ardner problem.
\newblock 2013.
\newblock available online at \bl{\url{http://arxiv.org/abs/1306.3979}}.

\bibitem{StojnicDiscPercp13}
M.~Stojnic.
\newblock Discrete perceptrons.
\newblock 2013.
\newblock available online at \bl{\url{http://arxiv.org/abs/1303.4375}}.

\bibitem{StojnicLiftStrSec13}
M.~Stojnic.
\newblock Lifting $\ell_1$-optimization strong and sectional thresholds.
\newblock 2013.
\newblock available online at \bl{\url{http://arxiv.org/abs/1306.3770}}.

\bibitem{StojnicMoreSophHopBnds10}
M.~Stojnic.
\newblock Lifting/lowering {H}opfield models ground state energies.
\newblock 2013.
\newblock available online at \bl{\url{http://arxiv.org/abs/1306.3975}}.

\bibitem{StojnicGorEx10}
M.~Stojnic.
\newblock Meshes that trap random subspaces.
\newblock 2013.
\newblock available online at \bl{\url{http://arxiv.org/abs/1304.0003}}.

\bibitem{StojnicGardSphNeg13}
M.~Stojnic.
\newblock Negative spherical perceptron.
\newblock 2013.
\newblock available online at \bl{\url{http://arxiv.org/abs/1306.3980}}.

\bibitem{StojnicRegRndDlt10}
M.~Stojnic.
\newblock Regularly random duality.
\newblock 2013.
\newblock available online at \bl{\url{http://arxiv.org/abs/1303.7295}}.

\bibitem{StojnicGardSphErr13}
M.~Stojnic.
\newblock Spherical perceptron as a storage memory with limited errors.
\newblock 2013.
\newblock available online at \bl{\url{http://arxiv.org/abs/1306.3809}}.

\bibitem{Stojnicgscompyx16}
M.~Stojnic.
\newblock Fully bilinear generic and lifted random processes comparisons.
\newblock 2016.
\newblock available online at \bl{\url{http://arxiv.org/abs/1612.08516}}.

\bibitem{Stojnicgscomp16}
M.~Stojnic.
\newblock Generic and lifted probabilistic comparisons -- max replaces minmax.
\newblock 2016.
\newblock available online at \bl{\url{http://arxiv.org/abs/1612.08506}}.

\bibitem{Stojnicsflgscompyx23}
M.~Stojnic.
\newblock Bilinearly indexed random processes -- {\emph{stationarization}} of
  fully lifted interpolation.
\newblock 2023.
\newblock available online at \bl{\url{http://arxiv.org/abs/2311.18097}}.

\bibitem{Stojnicbinperflrdt23}
M.~Stojnic.
\newblock Binary perceptrons capacity via fully lifted random duality theory.
\newblock 2023.
\newblock available online at \bl{\url{http://arxiv.org/abs/2312.00073}}.

\bibitem{Stojnictcmspnncaprdt23}
M.~Stojnic.
\newblock Capacity of the treelike sign perceptrons neural networks with one
  hidden layer -- rdt based upper bounds.
\newblock 2023.
\newblock available online at \bl{\url{http://arxiv.org/abs/2312.08244}}.

\bibitem{Stojnicnegsphflrdt23}
M.~Stojnic.
\newblock Fl rdt based ultimate lowering of the negative spherical perceptron
  capacity.
\newblock 2023.
\newblock available online at \bl{\url{http://arxiv.org/abs/2312.16531}}.

\bibitem{Stojnicnflgscompyx23}
M.~Stojnic.
\newblock Fully lifted interpolating comparisons of bilinearly indexed random
  processes.
\newblock 2023.
\newblock available online at \bl{\url{http://arxiv.org/abs/2311.18092}}.

\bibitem{Stojnicflrdt23}
M.~Stojnic.
\newblock Fully lifted random duality theory.
\newblock 2023.
\newblock available online at \bl{\url{http://arxiv.org/abs/2312.00070}}.

\bibitem{Stojnictcmspnncapliftedrdt23}
M.~Stojnic.
\newblock {\emph{Lifted}} rdt based capacity analysis of the 1-hidden layer
  treelike \emph{sign} perceptrons neural networks.
\newblock 2023.
\newblock available online at \bl{\url{http://arxiv.org/abs/2312.08257}}.

\bibitem{Stojnictcmspnncapdinfdiffactrdt23}
M.~Stojnic.
\newblock Exact capacity of the \emph{wide} hidden layer treelike neural
  networks with generic activations.
\newblock 2024.
\newblock available online at \bl{\url{http://arxiv.org/abs/2402.05719}}.

\bibitem{Stojnictcmspnncapdiffactrdt23}
M.~Stojnic.
\newblock Fixed width treelike neural networks capacity analysis -- generic
  activations.
\newblock 2024.
\newblock available online at \bl{\url{http://arxiv.org/abs/2402.05696}}.

\bibitem{Stojnicinjrelu24}
M.~Stojnic.
\newblock Injectivity capacity of relu gates.
\newblock 2024.
\newblock available online at \bl{\url{http://arxiv.org/abs/2410.20646}}.

\bibitem{Tal06}
M.~Talagrand.
\newblock The {P}arisi formula.
\newblock {\em Annals of mathematics}, 163(2):221--263, 2006.

\bibitem{Talbook11a}
M.~Talagrand.
\newblock {\em Mean field models and spin glasses: {V}olume {I}}.
\newblock A series of modern surveys in mathematics 54, Springer-Verlag, Berlin
  Heidelberg, 2011.

\bibitem{Ven86}
S.~Venkatesh.
\newblock Epsilon capacity of neural networks.
\newblock {\em Proc. Conf. on Neural Networks for Computing, Snowbird, UT},
  1986.

\bibitem{Wendel}
J.~G. Wendel.
\newblock A problem in geometric probablity.
\newblock {\em Mathematics Scandinavia}, 11:109--111, 1962.

\bibitem{Winder61}
R.~O. Winder.
\newblock Single stage threshold logic.
\newblock {\em Switching circuit theory and logical design}, pages 321--332,
  Sep. 1961.
\newblock AIEE Special publications S-134.

\bibitem{Winder}
R.~O. Winder.
\newblock {\em Threshold logic}.
\newblock Ph. D. dissertation, Princetoin University, 1962.

\bibitem{WuRL19}
Y.~Wu, M.~Rosca, and T.~P. Lillicrap.
\newblock Deep compressed sensing.
\newblock In {\em Proceedings of the 36th International Conference on Machine
  Learning, {ICML} 2019, 9-15 June 2019, Long Beach, California, {USA}},
  volume~97 of {\em Proceedings of Machine Learning Research}, pages
  6850--6860. {PMLR}, 2019.

\bibitem{ZavPeh21}
J.~A. Zavatone-Veth and C.~Pehlevan.
\newblock Activation function dependence of the storage capacity of treelike
  neural networks.
\newblock {\em Phys. Rev. E}, 103:L020301, February 2021.

\end{thebibliography}
\end{singlespace}


\end{document}